\documentclass{article}

 \usepackage[main, final]{neurips_2026}

\usepackage[utf8]{inputenc} 
\usepackage[T1]{fontenc}    
\usepackage{hyperref}       
\usepackage{url}            
\usepackage{booktabs}       
\usepackage{amsfonts}       
\usepackage{nicefrac}       
\usepackage{microtype}      
\usepackage{xcolor}         
\usepackage{graphicx}

\usepackage{wrapfig}

\usepackage{amsmath}
\usepackage{amssymb}

\usepackage{amsthm}
\usepackage{multirow}
\theoremstyle{plain}
\newtheorem{lemma}{Lemma}
\newtheorem{proposition}{Proposition}
\usepackage[table]{xcolor}
\theoremstyle{remark}

\title{Predictive but Not Plannable: RC-aux for Latent World Models}

\author{%
  Wenyuan Li \\
  Hokkaido University\\
  \texttt{wenyuan@lmd.ist.hokudai.ac.jp} \\
  \And
  Guang Li\thanks{Correspondence to: Guang Li <guang@lmd.ist.hokudai.ac.jp>} \\
  Hokkaido University \\
  \texttt{guang@lmd.ist.hokudai.ac.jp} \\
  \And
  Keisuke Maeda\\
  Hokkaido University \\
  \texttt{maeda@lmd.ist.hokudai.ac.jp} \\
  \And
  Takahiro Ogawa \\
  Hokkaido University \\
  \texttt{ogawa@lmd.ist.hokudai.ac.jp} \\
  \And
  Miki Haseyama \\
  Hokkaido University \\
  \texttt{mhaseyama@lmd.ist.hokudai.ac.jp} \\
}

\begin{document}

\maketitle

\begin{abstract}
A latent world model may achieve accurate short-horizon prediction while still inducing a latent space that is poorly aligned with planning. 
A key issue is spatiotemporal mismatch: these models are often trained with local predictive supervision, but deployed for long-horizon goal-directed search in latent spaces where Euclidean distance may not reflect what is reachable within a finite action budget.
We present the \textbf{\underline{R}}eachability-\textbf{\underline{C}}orrection \textbf{\underline{aux}}iliary objective (RC-aux), a lightweight correction for this mismatch in reconstruction-free latent world models. 
RC-aux keeps the world-model backbone unchanged and adds planning-aligned supervision along two axes. 
Along the \textit{time} axis, multi-horizon open-loop prediction trains the model beyond one-step consistency. 
Along the \textit{space} axis, budget-conditioned reachability supervision, together with temporal hard negatives, encourages the latent space to distinguish states that are eventually reachable from those reachable within the current planning horizon. 
At test time, the learned reachability signal can also be used by a reachability-aware planner to favor trajectories that are both goal-directed and attainable under the available budget.
We instantiate RC-aux on LeWorldModel and evaluate it under both continuation-training and matched-from-scratch settings. 
Across goal-conditioned pixel-control tasks and a LIBERO-Goal extension, RC-aux improves LeWM-style planning with modest additional cost. 
These results suggest that planning with latent world models depends not only on predictive accuracy, but also on whether the learned representation encodes the temporal and geometric structure required by downstream search. The code is available at \url{https://github.com/Guang000/RC-aux}.
\end{abstract}

\section{Introduction}

A latent world model can be predictive without being plannable.
In goal-conditioned control from pixels, planning requires more than accurate latent prediction.
The planner also relies on the latent space to rank candidate trajectories: terminal latent distance should reflect feasible progress toward the goal, and the planning horizon should reflect what can be achieved with a finite action budget.
A candidate rollout may end near the goal in Euclidean latent distance while following a shortcut that is not supported by feasible finite-horizon transitions.
When the latent geometry permits such shortcuts, a planner can optimize a plausible latent objective while selecting trajectories that are short in representation space but poor in the environment.
Figure~\ref{fig:overview} illustrates this failure mode and the correction introduced by our method.

Latent world models have become a central tool for control from pixels, enabling agents to plan or learn behavior through compact latent rollouts~\citep{hafner2019planet,hafner2023dreamerv3,hansen2024tdmpc2}.
Recent reconstruction-free and JEPA-style models make this especially attractive for offline and goal-conditioned control, since they avoid pixel-level decoding while still supporting planning in learned latent spaces~\citep{zhou2025dinowm,sobal2025pldm,maes2026leworldmodel}.
However, many such models are trained under objectives that are only indirectly aligned with their test-time use.
A model trained primarily for short-horizon latent prediction is later embedded in a long-horizon planner that optimizes terminal latent distance to a goal.
This train-test gap makes prediction accuracy alone insufficient for reliable planning.

\begin{wrapfigure}{r}{0.52\textwidth}
    \centering
    \includegraphics[width=\linewidth]{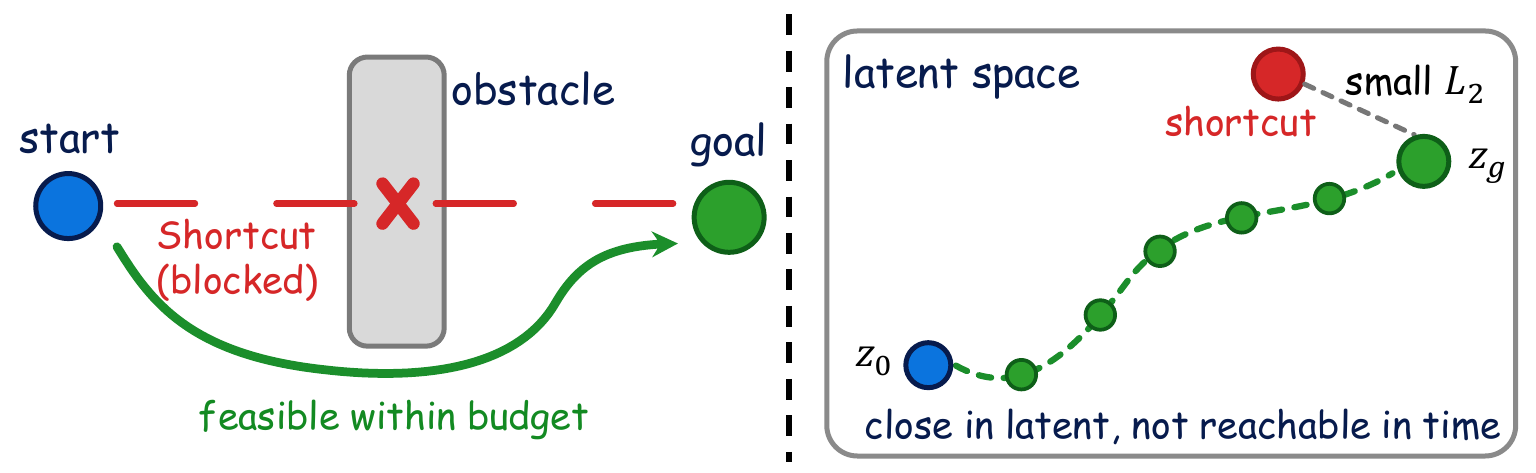}
\caption{
Conceptual illustration of the latent-shortcut failure mode.
A terminal latent-distance planner may favor a shortcut that is close in latent space but unreachable within the finite action budget.
RC-aux encourages distance to align with finite-horizon reachability, making feasible trajectories more consistent with planning.
}
\label{fig:overview}
\end{wrapfigure}

The gap creates two coupled mismatches.
The temporal mismatch arises because training supervision is concentrated on short-horizon prediction, whereas the planner must evaluate many-step open-loop rollouts at test time.
The spatial mismatch arises because Euclidean proximity in latent space does not necessarily reflect the finite-horizon reachability structure supported by the data.
As a result, planning can fail even when one-step prediction is accurate: the planner may optimize a latent objective whose geometry does not distinguish feasible progress from unsupported latent-space shortcuts.
Recent work on value-guided representations, temporal straightening, plannable continuous latents, and quasimetric reachability similarly shows that planning depends on representation geometry, not only prediction error~\citep{destrade2026valueguidedjepa,wang2026temporalstraightening,koul2024pclast,wang2023qrl,qian2023replan,bae2024tldr}.
These approaches motivate planning-aware geometry, but they do not directly align the open-loop rollout distribution, finite-horizon reachability structure, and terminal objective used by latent planners.

We propose a Reachability-Correction auxiliary objective (RC-aux) for latent world models, together with a reachability-aware planner that uses the learned signal during search.
RC-aux makes the training queries more closely match the queries made by the planner.
It corrects the temporal mismatch by training the dynamics model on multi-horizon open-loop rollouts.
It corrects the spatial mismatch by learning a budget-conditioned reachability predicate that estimates whether a target latent is attainable from a source latent within a finite horizon.
The reachability objective uses reachable positives, batch negatives, and same-trajectory \emph{temporal hard negatives}: future states that are eventually attainable but not reachable under the current budget.
These hard negatives make the budget dependence identifiable by forcing the model to distinguish ``reachable eventually'' from ``reachable within this horizon.''
We also impose the same reachability structure on predicted rollout latents, so states generated during planning inherit the finite-horizon semantics used by the planner.

At test time, our planner couples this learned reachability signal with terminal latent planning.
Standard terminal latent L2 planning is recovered as the $\lambda_{\mathrm{plan}}=0$ member of this planner family.
For $\lambda_{\mathrm{plan}}>0$, the planner favors candidate rollouts that are close to the goal under the base latent cost and contain intermediate predicted states from which the goal is estimated reachable within the remaining horizon.
Thus, reachability becomes an explicit search signal rather than only a training-time regularizer.

We instantiate RC-aux on LeWorldModel (LeWM), a compact reconstruction-free JEPA world model for goal-conditioned planning from pixels~\citep{maes2026leworldmodel}.
RC-aux preserves the LeWM backbone, allowing us to isolate the effect of correcting the training horizon, shaping finite-horizon reachability geometry, and using the learned reachability signal during search.
Although our experiments use LeWM, the objective itself is backbone-agnostic: it requires only latent rollouts, goal latents, and finite-horizon transition structure.

Our experiments evaluate RC-aux on the original LeWM task family and benchmark extensions under matched local protocols, including continuation controls, paired fixed-episode outcomes, ablations, and rollout visualizations.
Across these settings, RC-aux improves four out of five matched LeWM-family comparisons, with the clearest gain on Wall, where finite-budget reachability makes latent Euclidean distance a particularly weak proxy for planning progress.
These results support the central claim of this paper: strong latent planning requires not only accurate prediction, but a representation whose notions of distance and horizon are aligned with data-supported finite-horizon reachability.

Our contributions are:

\begin{itemize}

\item We introduce RC-aux, a backbone-agnostic auxiliary objective that makes latent world models more plannable by jointly correcting rollout-horizon mismatch and finite-horizon reachability geometry.

\item We introduce a reachability-aware planner family that uses the learned reachability signal during test-time search, with standard terminal latent L2 planning recovered when $\lambda_{\mathrm{plan}}=0$.

\item We provide a formal analysis showing how multi-horizon open-loop prediction, budget-conditioned reachability, and temporal hard negatives align training with finite-horizon planning queries.

\item We provide controlled empirical validation on LeWorldModel, including matched comparisons, continuation controls, ablations, rollout visualizations, and a larger LIBERO-Goal extension with OFT-style action heads.

\end{itemize}

\section{Related Work}
\label{sec:related}

RC-aux builds on latent world models for pixel-based control, including recurrent latent planners, imagined-rollout agents, and decoder-free control-centric models~\citep{hafner2019planet,hafner2023dreamerv3,hansen2024tdmpc2}. 
It is closest to reconstruction-free and JEPA-style models such as DINO-WM, PLDM, and LeWorldModel, which avoid pixel reconstruction and plan directly in learned or pretrained visual feature spaces~\citep{zhou2025dinowm,sobal2025pldm,maes2026leworldmodel}. 
Rather than introducing a new backbone, RC-aux studies how to train the latent space so that it better exposes the finite-horizon structure needed by planning.

RC-aux is also related to planning-aware representation learning, including value-guided latent distances, temporally structured representations, plannable latent spaces, reachability learning, and quasimetric objectives~\citep{destrade2026valueguidedjepa,wang2026temporalstraightening,koul2024pclast,wang2023qrl,qian2023replan,bae2024tldr,steccanella2022state}. 
These methods show that latent proximity or short-step prediction error alone may be insufficient for control. 
RC-aux targets this mismatch with budget-conditioned reachability supervision, temporal hard negatives, and multi-horizon open-loop prediction. 
Additional discussion is provided in Appendix~\ref{app:extended_related}.

\section{Method}
\label{sec:method}

We consider goal-conditioned control from pixels with a latent world model.
The offline dataset consists of trajectories
\begin{equation}
    \mathcal D
    =
    \{\tau^{(n)}\}_{n=1}^{N},
    \qquad
    \tau=(o_{1:T},a_{1:T-1}),
\end{equation}
where $o_t$ is an image observation and $a_t$ is the action between $o_t$ and $o_{t+1}$.
A latent world model consists of an encoder $e_\theta$ and an action-conditioned latent dynamics model.
The encoder maps observations to latent states,
\begin{equation}
    z_t=e_\theta(o_t),
\end{equation}
and the dynamics model predicts future latents under candidate action sequences.
Given a current observation $o_t$ and a goal observation $o_g$, a latent planner encodes
\begin{equation}
    z_t=e_\theta(o_t),
    \qquad
    z_g=e_\theta(o_g),
\end{equation}
rolls out candidate action sequences in latent space, and selects the sequence with the lowest goal-matching cost.


\subsection{Time alignment: open-loop multi-horizon prediction}
\label{sec:method_time}

A planner evaluates a world model by feeding predicted latents back into the dynamics model.
RC-aux therefore trains the model in the same open-loop regime.
For a sampled trajectory segment, let $z_{t-L+1:t}$ be the context latents and let $a_{t:t+K-1}$ be the future actions used for rollout.
Any action history used by the underlying backbone is absorbed into the context notation.
The model produces an open-loop prediction
\begin{equation}
    \hat z_{t+1:t+K}
    =
    F_\theta^{(K)}
    \left(
        z_{t-L+1:t},
        a_{t:t+K-1}
    \right),
    \label{eq:mh_rollout}
\end{equation}
where predicted latents are fed back into the dynamics model for later steps.
The encoded future latents
\begin{equation}
    z_{t+1},z_{t+2},\ldots,z_{t+K}
\end{equation}
are used as supervision targets, not as rollout inputs.
For this segment, the multi-horizon open-loop loss is
\begin{equation}
    \ell_{\mathrm{mh}}
    =
    \sum_{k=1}^{K}
        w_k
        \left\|
            \hat z_{t+k}
            -
            z_{t+k}
        \right\|_2^2,
    \label{eq:mh_loss}
\end{equation}
where $w_k$ are horizon weights.
The context latent $z_t$ is the rollout start; the first supervised prediction is $\hat z_{t+1}$, which is matched to $z_{t+1}$.
Thus, the loss directly trains the model on the predicted latents that will later be scored by the planner.

\begin{figure}[!t]    \centerline{\includegraphics[width=1\linewidth]{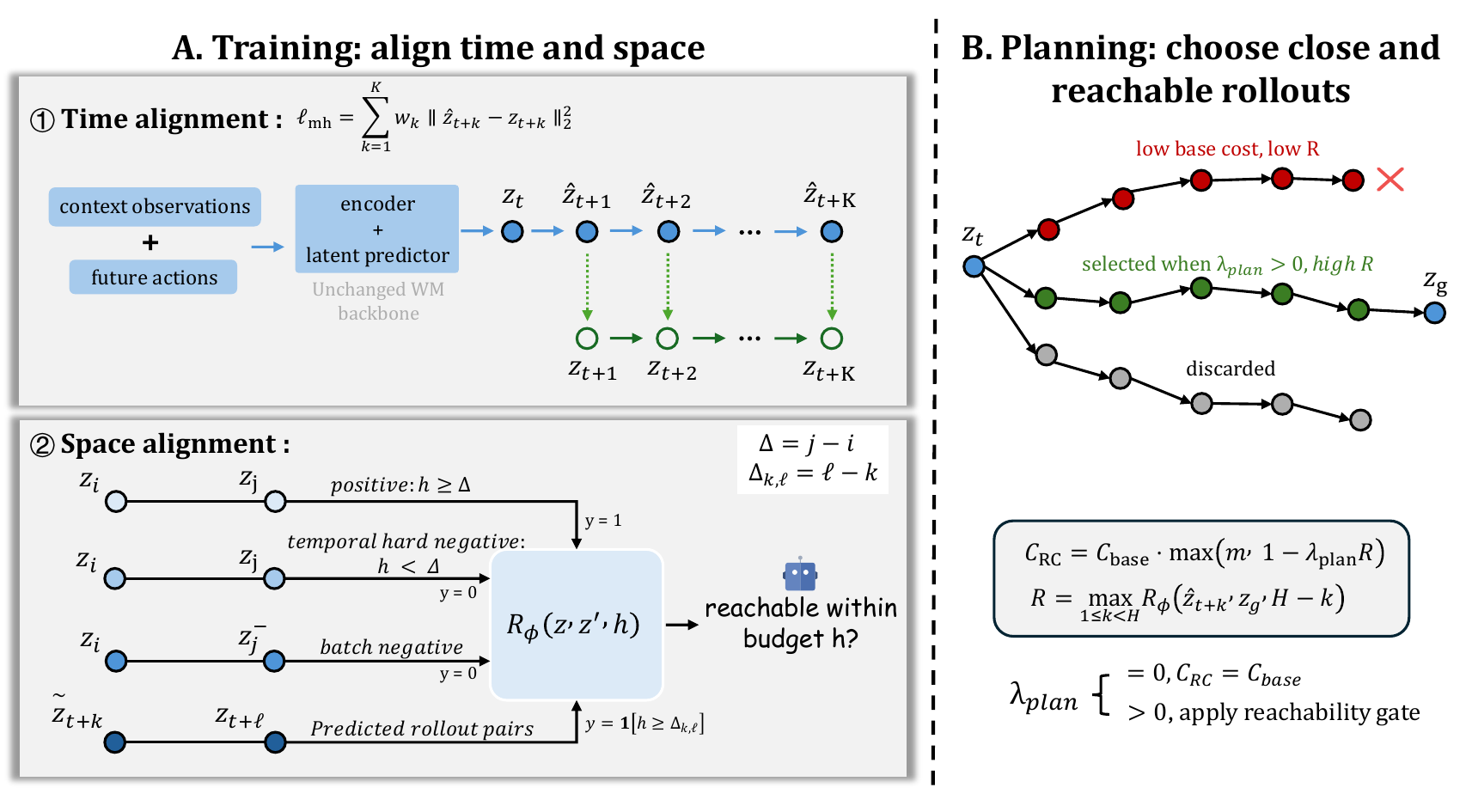}}
\caption{
The framework of RC-aux.
RC-aux keeps the latent world-model backbone unchanged and aligns it with planning through two training signals:
multi-horizon open-loop prediction for time alignment, and budget-conditioned reachability $R_\phi(z,z',h)$ for space alignment.
Reachability is trained with positives, trajectory-induced hard negatives, batch negatives, and predicted rollout pairs.
At test time, $\lambda_{\mathrm{plan}}=0$ recovers the base latent planner, while $\lambda_{\mathrm{plan}}>0$ gates the base goal cost by trajectory-level reachability to favor rollouts that are both close to the goal and empirically attainable.
}
\label{method}
\end{figure}

\subsection{Space alignment: budget-conditioned reachability}
\label{sec:method_reach}

Multi-step prediction alone does not guarantee that latent distance reflects attainability.
RC-aux therefore learns a budget-conditioned reachability predicate
\begin{equation}
    R_\phi(z,z',h)
    =
    \sigma(r_\phi(z,z',h))
    \in[0,1],
    \label{eq:reach_head}
\end{equation}
where $R_\phi(z,z',h)$ estimates whether target latent $z'$ is reachable from source latent $z$ within a finite budget $h$.
The reachability head is conditioned on an ordered source--target pair and a horizon budget.
The order matters because finite-horizon reachability is generally asymmetric.

For same-trajectory pairs $(z_i,z_j)$ with $i<j$, define the observed trajectory offset
\begin{equation}
    \Delta=j-i.
\end{equation}
RC-aux uses trajectory offsets to construct empirical finite-budget reachability labels.
These labels should be interpreted as \emph{trajectory-induced supervision} rather than ground-truth shortest-path reachability in the underlying MDP: a target that is observed after $\Delta$ steps along one trajectory may in principle be reachable by a shorter unobserved path.
Nevertheless, such offsets provide a useful proxy for the finite-horizon attainability structure that a planner needs.

For each ordered pair, budgets are sampled from
\begin{equation}
    h\in\{0,1,\ldots,H_{\max}\}.
\end{equation}
The trajectory-induced label is positive when the target is observed within the sampled budget along the trajectory, and negative otherwise:
\begin{equation}
    y_{ijh}
    =
    \mathbf 1[h\ge\Delta].
    \label{eq:reach_label}
\end{equation}

\paragraph{Reachable positives.}
If $h\ge\Delta$, then $z_j$ is observed within the available budget after $z_i$ along the sampled trajectory.
The pair $(z_i,z_j,h)$ is labeled reachable.

\paragraph{Trajectory-induced temporal hard negatives.}
If $h<\Delta$, then $z_j$ is not observed to be reached from $z_i$ within budget $h$ along the sampled trajectory.
The pair $(z_i,z_j,h)$ is therefore used as a trajectory-induced temporal hard negative.
These examples force the reachability head to distinguish eventual co-occurrence along a trajectory from empirical finite-budget attainability.

\paragraph{Batch negatives.}
Targets sampled from other trajectories in the minibatch are used as cross-trajectory negatives.
These negatives prevent the reachability head from declaring arbitrary latent pairs attainable.
For a batch negative $(z_i,z_j^{-},h)$, we assign label
\begin{equation}
    y_{ijh}^{-}=0.
\end{equation}

\paragraph{Predicted rollout pairs.}
The planner queries reachability on predicted latents, not only on encoded dataset latents.
RC-aux therefore also trains the reachability head on predicted-latent inputs produced by open-loop rollouts.
For the open-loop rollout in Eq.~\eqref{eq:mh_rollout}, predicted latents used in reachability supervision are treated as stop-gradient inputs unless otherwise stated:
\begin{equation}
    \tilde z_{t+k}
    =
    \operatorname{sg}(\hat z_{t+k}).
\end{equation}
For predicted-source pairs, RC-aux pairs a predicted intermediate latent $\tilde z_{t+k}$ with an encoded future target $z_{t+\ell}$, where $0<k<\ell\le K$.
The empirical offset from the predicted source to the encoded target is
\begin{equation}
    \Delta_{k,\ell}
    =
    \ell-k,
\end{equation}
and the corresponding finite-budget label is
\begin{equation}
    y_{k\ell h}
    =
    \mathbf 1[h\ge\Delta_{k,\ell}].
    \label{eq:pred_reach_label}
\end{equation}
This construction matches the planner-time query: from a predicted intermediate latent, the model asks whether a future target can be reached within the remaining budget.
The encoded-latent reachability terms provide representation-shaping gradients to the encoder, while the stop-gradient predicted-latent terms calibrate the reachability head on planner-induced latent distributions.

\paragraph{Avoiding budget-agnostic shortcuts.}
Trajectory-induced temporal hard negatives make the budget variable identifiable.
Without them, all same-trajectory pairs observed by the reachability loss would be positive, while batch negatives would be negative.
In that case, a classifier could fit the supervision by simply distinguishing same-trajectory pairs from cross-trajectory pairs, without using the budget $h$.
Temporal hard negatives remove this degenerate solution.
For the same ordered pair $(z_i,z_j)$ with offset $\Delta=j-i$, the label is negative when $h<\Delta$ and positive when $h\ge\Delta$.
Therefore, any classifier that fits these labels must assign different predictions to the same latent pair under different budgets.
This forces $R_\phi(z,z',h)$ to represent empirical finite-horizon attainability rather than merely visual similarity or trajectory membership.

\paragraph{Planning-alignment view.}
Ideally, a finite-horizon planner would benefit from a predicate that indicates whether a goal is reachable within the remaining budget.
Let $D^\star(s,g)$ denote the minimum number of environment steps required to reach $g$ from $s$, with $D^\star(s,g)=\infty$ if $g$ is unreachable.
The ideal finite-budget reachability predicate is
\begin{equation}
    R_h^\star(s,g)
    =
    \mathbf 1[D^\star(s,g)\le h].
\end{equation}
RC-aux does not assume access to $D^\star$.
Instead, it uses observed trajectory offsets as an empirical proxy for this horizon-indexed predicate.
This proxy is directed and budget-dependent, matching the structure of finite-horizon planning more closely than a symmetric, budget-free Euclidean latent distance. Appendix B further formalizes this view by analyzing open-loop cost distortion, budget identifiability, and the effect of reachability-aware planning as a soft feasibility gate.

Let $\mathcal B_{\mathrm{enc}}$ be the set of reachability pairs constructed from encoded latents and $\mathcal B_{\mathrm{pred}}$ the set of analogous pairs involving stop-gradient predicted rollout latents.
For a pair $(z,z',h,y)$ with label $y\in\{0,1\}$, define
\begin{equation}
    \ell_{\mathrm{bce}}(z,z',h,y)
    =
    \mathrm{BCE}
    \left(
        R_\phi(z,z',h),
        y
    \right).
\end{equation}
The reachability loss for a minibatch is
\begin{equation}
\begin{aligned}
    \ell_{\mathrm{reach}}
    &=
    \frac{1}{|\mathcal B_{\mathrm{enc}}|}
    \sum_{(z,z',h,y)\in\mathcal B_{\mathrm{enc}}}
    \omega_y\,
    \ell_{\mathrm{bce}}(z,z',h,y)
    \\
    &\quad
    +
    \rho_{\mathrm{pred}}\,
    \frac{1}{|\mathcal B_{\mathrm{pred}}|}
    \sum_{(z,z',h,y)\in\mathcal B_{\mathrm{pred}}}
    \omega_y\,
    \ell_{\mathrm{bce}}(z,z',h,y),
\end{aligned}
\label{eq:reach_loss}
\end{equation}
where $\omega_y$ are optional class-balancing weights and $\rho_{\mathrm{pred}}$ controls the weight on predicted-latent reachability supervision.

\subsection{RC-aux training objective}
\label{sec:method_objective}

RC-aux preserves the underlying world-model backbone and its latent regularization, but replaces local prediction supervision with planning-aligned open-loop supervision.
For a sampled trajectory segment, the core RC-aux objective is
\begin{equation}
    \ell_{\mathrm{RC\text{-}aux}}
    =
    \ell_{\mathrm{mh}}
    +
    \alpha \ell_{\mathrm{reg}}
    +
    \beta \ell_{\mathrm{reach}},
    \label{eq:rcaux_objective}
\end{equation}
where $\ell_{\mathrm{mh}}$ is the multi-horizon open-loop prediction loss, $\ell_{\mathrm{reach}}$ is the budget-conditioned reachability loss, and $\ell_{\mathrm{reg}}$ denotes the latent regularizer of the underlying world model.
For the LeWM instantiation, $\ell_{\mathrm{reg}}$ is the original SIGReg regularizer, while $\ell_{\mathrm{mh}}$ replaces the original one-step latent prediction loss rather than being added on top of it.
During training, Eq.~\eqref{eq:rcaux_objective} is averaged over sampled trajectory segments.
Thus, RC-aux is not a new world-model architecture and is not merely an additional test-time planner; it changes what the learned latent dynamics and latent geometry are trained to support: open-loop prediction over planning horizons and finite-budget reachability.

\subsection{Planner family}
\label{sec:method_planner}

At evaluation time, RC-aux uses the same latent rollout interface as the base planner.
For a candidate action sequence
\begin{equation}
    \tau=a_{t:t+H-1},
\end{equation}
the model predicts
\begin{equation}
    \hat z_{t+1:t+H}
    =
    F_\theta^{(H)}
    \left(
        z_{t-L+1:t},
        \tau
    \right).
\end{equation}
Let $C_{\mathrm{base}}(\tau)$ denote the base latent goal cost.
For terminal latent planning,
\begin{equation}
    C_{\mathrm{base}}(\tau)
    =
    \left\|
        \hat z_{t+H}
        -
        z_g
    \right\|_2^2.
\end{equation}
Other reductions over the predicted rollout, such as minimum, soft-minimum, or mean distance, can also be used.

RC-aux defines a planner family parameterized by a reachability coupling weight
\begin{equation}
    \lambda_{\mathrm{plan}}\in[0,1].
\end{equation}
For each predicted intermediate latent $\hat z_{t+k}$, the reachability head asks whether the goal can be reached within the remaining budget $H-k$.
We use intermediate latents with positive remaining budget:
\begin{equation}
    1\le k < H.
\end{equation}
The trajectory-level reachability score is
\begin{equation}
    R(\tau,z_g)
    =
    \max_{1\le k < H}
    R_\phi
    \left(
        \hat z_{t+k},
        z_g,
        H-k
    \right).
    \label{eq:traj_reach}
\end{equation}
The reachability-aware planning cost is
\begin{equation}
    C_{\mathrm{RC}}(\tau)
    =
    C_{\mathrm{base}}(\tau)
    \cdot
    \max
    \left(
        m,\,
        1-\lambda_{\mathrm{plan}}R(\tau,z_g)
    \right),
    \label{eq:rc_planner}
\end{equation}
where $m>0$ is a small floor for numerical stability.

When $\lambda_{\mathrm{plan}}=0$, Eq.~\eqref{eq:rc_planner} recovers the base latent planner exactly.
When $\lambda_{\mathrm{plan}}>0$, the reachability score softly discounts the base cost for rollouts whose intermediate predicted latents are estimated to make the goal reachable within the remaining budget.
Low reachability does not add a separate penalty; it simply removes this discount.
Thus, the planner still respects the base latent goal cost while favoring trajectories that are both geometrically close and empirically attainable under the learned reachability predicate.
This planner-side use is optional: $\lambda_{\mathrm{plan}}=0$ isolates the training-side effect of RC-aux, while $\lambda_{\mathrm{plan}}>0$ evaluates reachability-aware planning.

\section{Experiments}
\label{sec:experiments}

\subsection{Experimental Setup}
\label{sec:exp_setup}

We evaluate RC-aux on five pixel-based goal-conditioned control tasks: TwoRoom~\citep{two}, Reacher~\citep{reacher}, Push-T~\citep{zhou2025dinowm}, Wall~\citep{zhou2025dinowm}, and Cube~\citep{ogbench}.
Each episode provides an initial observation and a goal image, and we report success rate as the primary metric.
We compare RC-aux with prior baselines and LeWM-family controls under the same evaluation protocol.
For local LeWM-family runs, results are averaged over five fixed evaluation groups; when available, we include a continuation-trained LeWM control to account for additional training under the same backbone.
We additionally evaluate representation transfer on LIBERO-Goal~\citep{libero} by training an OFT-style action head on top of the learned representation.
Full evaluation details, planner settings, and training hyperparameters are provided in Appendix~\ref{app:exp_details}. All experiments were conducted on a single NVIDIA RTX 6000 Ada GPU.

\subsection{Main Results}
\label{sec:main_results}

\begin{wraptable}[18]{r}{0.58\textwidth}
\vspace{-0.8em}
\centering
\caption{
Success rates on five pixel-based goal-conditioned control tasks.
Local LeWM-family results are reported as mean$\pm$std
over five fixed evaluation groups.
The matched $\Delta$ row compares RC-aux against LeWM-cont.
}
\label{tab:main-success}
\resizebox{\linewidth}{!}{
\begin{tabular}{lccccc}
\toprule
Method & TwoRoom & Reacher & Push-T & Wall & Cube \\
\midrule
DINO-WM~\citep{zhou2025dinowm} & 100.0 & 79.0 & 74.0 & 96.0 & 86.0 \\
PLDM~\citep{two} & 97.0 & 78.0 & 78.0 & -- & 65.0 \\
DINO-WM+prop~\citep{zhou2025dinowm} & -- & -- & 92.0 & -- & -- \\
GCBC~\citep{GCBC} & 100.0 & -- & 75.0 & -- & 84.0 \\
IQL~\citep{IQL} & 100.0 & -- & 20.0 & -- & 64.0 \\
IVL~\citep{ogbench} & 100.0 & -- & 33.0 & -- & 56.0 \\
\midrule
LeWM~\citep{maes2026leworldmodel} & 88.8$\pm$3.0 & 81.2$\pm$7.9 & 90.4$\pm$3.0 & 50.4$\pm$6.5 & 72.4$\pm$5.9 \\
LeWM-cont~\citep{maes2026leworldmodel} & 88.8$\pm$3.0 & 82.8$\pm$7.2 & \textbf{91.2$\pm$3.9} & 52.8$\pm$4.8 & 72.8$\pm$5.2 \\
\rowcolor{gray!30}
\textbf{RC-aux} & \textbf{98.0$\pm$1.4} & \textbf{87.2$\pm$6.4} & 90.8$\pm$3.3 & \textbf{83.6$\pm$3.6} & \textbf{76.0$\pm$7.5} \\
\midrule
Matched $\Delta$ 
& \cellcolor{red!50} +9.2 
& \cellcolor{red!30} +4.4 
& \cellcolor{blue!5} -0.4 
& \cellcolor{red!65} +30.8 
& \cellcolor{red!15} +3.2 \\
\bottomrule
\end{tabular}
}
\vspace{-0.8em}
\end{wraptable}

Table~\ref{tab:main-success} reports success rates on the five pixel-based goal-conditioned control tasks.
We include prior baselines for benchmark context and local LeWM-family runs for controlled comparison.
RC-aux achieves strong performance across the suite, with the clearest improvement on Wall, where success increases from $52.8\%$ for LeWM to $83.6\%$.
It also improves over the continuation-trained LeWM control on TwoRoom, Reacher, and Cube, reaching $98.0\%$, $87.2\%$, and $76.0\%$ success, respectively.
On Push-T, all LeWM-family methods are already in a high-success regime, and RC-aux remains competitive with the strongest local control.

The matched $\Delta$ row of Table~\ref{tab:main-success}
summarizes the within-family comparison against LeWM-cont
under the same task protocol across all five tasks.
RC-aux improves four out of five matched comparisons.
The largest gain appears on Wall, consistent with the
motivation of RC-aux: latent Euclidean proximity is a
particularly poor proxy for finite-horizon reachability
in obstacle-constrained planning.

Figure~\ref{fig:main-success} visualizes the same comparison.
The bar chart highlights both the broader benchmark context and the local LeWM-family comparison.
RC-aux is most effective on Wall and remains strong across the remaining tasks, supporting the view that reachability correction is most useful when the planner must reason about what can be reached within a finite action budget. Appendix~\ref{app:additional_results} provides paired fixed-episode outcomes and additional local success visualizations. 

\paragraph{Qualitative rollouts.}
Qualitative rollout visualizations are provided in
Appendix~\ref{app:additional_qualitative}.
Representative Cube and Wall examples illustrate how RC-aux
maintains intermediate progress toward the goal, complementing
the quantitative results in Table~\ref{tab:main-success}.

\begin{figure}[t]
    \centering
     \includegraphics[width=1.0\linewidth]{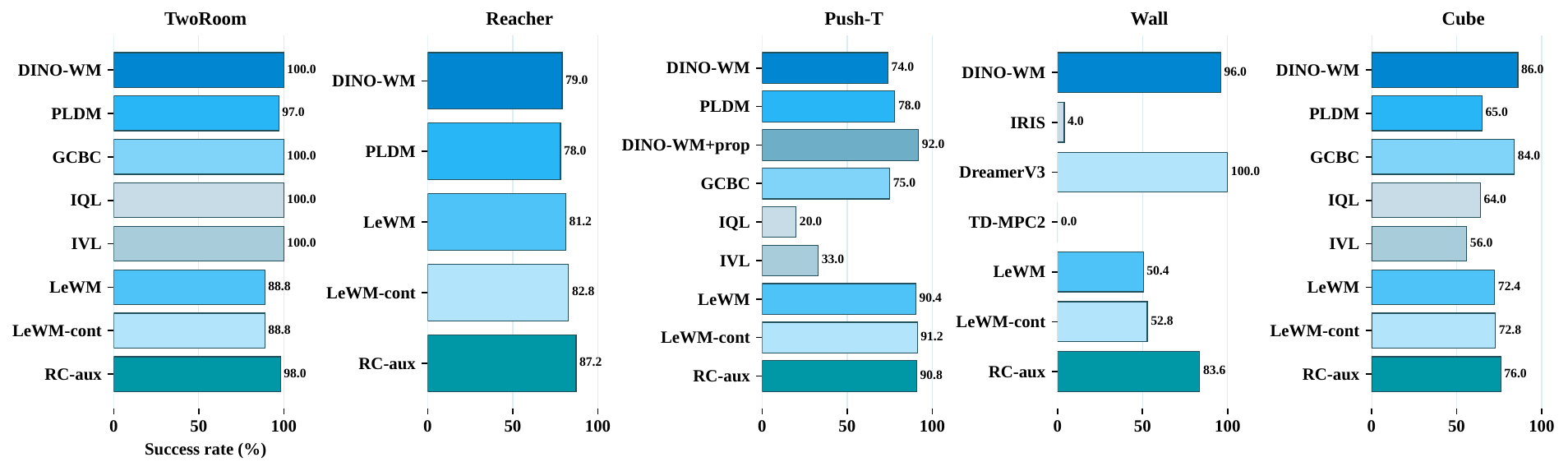}
    \caption{
    Success rates across the five pixel-based control tasks.
    RC-aux improves four of the five matched LeWM-family comparisons and gives the largest gain on Wall.
    }
    \label{fig:main-success}
\end{figure}

\subsection{Ablation Study}
\label{sec:ablation}

We examine the contributions of planner-side scoring, time and
space alignment, and explicit budget conditioning.

\begin{wraptable}{r}{0.5\textwidth}
    \centering
    \caption{
        Planner-side ablation using the same RC-aux model.
        Success rates are percentages.
    }
    \label{tab:planner_ablation}
    \small
    \setlength{\tabcolsep}{6pt}
    \renewcommand{\arraystretch}{1.05}
    \scalebox{0.90}{
        \begin{tabular}{lrr}
            \toprule
            Task & Planner off & Planner on \\
            \midrule
            TwoRoom & 93.2 & \textbf{98.0} \\
            Reacher & 81.2 & \textbf{87.2} \\
            \bottomrule
        \end{tabular}
    }
\end{wraptable}

\paragraph{Reachability-aware planning.}
We compare the same RC-aux-trained model with and without
reachability-aware scoring. Setting $\lambda_{\mathrm{plan}}=0$
recovers the original terminal latent-distance planner.
For planner-on evaluation, we select $\lambda_{\mathrm{plan}}$
from $\{0,0.35,0.85\}$ on validation episodes disjoint from the
fixed test groups, then freeze it for testing.
All other planner settings remain unchanged within each task.
As shown in Table~\ref{tab:planner_ablation}, reachability-aware
scoring improves success by 4.8 points on TwoRoom and 6.0 points
on Reacher, demonstrating its utility beyond the base latent
goal cost.

\paragraph{Time and space alignment.}
We compare a matched continued-training control, time alignment
alone, space alignment alone, and full RC-aux.
All variants retain the base latent regularizer.
The control uses the original training objective;
time alignment alone replaces one-step prediction with
multi-horizon prediction; space alignment alone adds reachability
supervision to the original one-step objective; and full RC-aux
combines both modifications.
Data, initialization procedures, and optimization budgets are
matched, and all variants use the same evaluation episodes with
$\lambda_{\mathrm{plan}}=0$.

Table~\ref{tab:training_ablation} shows task-dependent contributions.
Space alignment accounts for most of the TwoRoom improvement
(84.8\% to 92.4\%), with full RC-aux reaching 93.2\%.
On Wall, time alignment increases success from 52.8\% to 69.6\%,
while full RC-aux reaches 74.8\%.
The full objective achieves the highest observed success on
both tasks, although its margin over space alignment alone is
small on TwoRoom.
Its 22.0-point gain over the Wall continuation control also
indicates that additional training alone does not explain
the improvement.

\begin{table}[t]
    \centering
    \caption{
        Training-side ablations with
        $\lambda_{\mathrm{plan}}=0$.
        Success rates are percentages.
        MH denotes multi-horizon prediction.
    }
    \label{tab:training_ablation}
    \small
    \setlength{\tabcolsep}{6pt}
    \renewcommand{\arraystretch}{1.05}
    \scalebox{1}{
        \begin{tabular}{lrr}
            \toprule
            Training variant & TwoRoom & Wall \\
            \midrule
            Continued-training control & 84.8 & 52.8 \\
            Time alignment only        & 86.4 & 69.6 \\
            Space alignment only       & 92.4 & 60.0 \\
            Full RC-aux                & \textbf{93.2}
                                       & \textbf{74.8} \\
            \midrule
            MH + directed distance     & 87.6 & 63.6 \\
            \bottomrule
        \end{tabular}
    }
\end{table}

\paragraph{Explicit budget conditioning.}
We compare $R_{\phi}(z,z',h)$ with a budget-free directed-distance
model $d_{\psi}(z,z')$ trained to regress trajectory offsets.
The latter can represent a reachability decision through
$\widetilde{R}(z,z',h)=
\mathbf{1}[d_{\psi}(z,z')\leq h]$.
Both variants retain multi-horizon prediction and use matched
backbones, initialization, training budgets, and evaluation
protocols.
Neither auxiliary head nor the threshold decision is used
during planning; both variants use the original terminal
latent-distance planner.
The comparison therefore evaluates the effect of training
supervision on downstream planning.
Explicit budget conditioning outperforms distance regression
by 5.6 points on TwoRoom and 11.2 points on Wall
(Table~\ref{tab:training_ablation}), supporting the benefit
of directly supervising finite-budget reachability.
Additional paired outcomes and local success visualizations
are provided in Appendix~\ref{app:additional_results}.

\subsection{Model Size and Computational Overhead}
\label{app:efficiency}

RC-aux preserves the LeWM backbone and adds only a lightweight reachability
head.  Table~\ref{tab:param-counts} reports the parameter footprint of the
modules used during planner scoring.  RC-aux increases the parameter count from
$18.034$M to $18.710$M, corresponding to only a $3.74\%$ overhead over LeWM.

\begin{wraptable}{r}{0.4\textwidth}
\centering
\caption{
Parameter footprint of modules used during planner scoring.
For DINO-WM-S/14, the frozen DINOv2-S/14 encoder is included because it is
active at scoring time.
}
\label{tab:param-counts}
\scalebox{0.65}{
\begin{tabular}{lcc}
\toprule
Model & parameters & Relative to RC-aux \\
\midrule
DINO-WM-S/14~\citep{zhou2025dinowm} & 42.179M & $2.25\times$ \\
LeWM~\citep{maes2026leworldmodel} & 18.034M & $0.96\times$ \\
\rowcolor{gray!30} \textbf{RC-aux} & 18.710M & $1.00\times$ \\
\bottomrule
\end{tabular}
}
\end{wraptable}

For comparison with a foundation-model-based world model, we also include
DINO-WM-S/14.  Unlike RC-aux, DINO-WM uses a frozen DINOv2-S/14 visual encoder
pretrained on external image data.  Although this encoder is frozen, it is still
used during planner scoring, so we include it in the scoring-time parameter
footprint.  Under this accounting, DINO-WM-S/14 has $42.179$M active
scoring-time parameters, about $2.25\times$ the size of RC-aux.

We also measure controlled planner cost-call overhead.  This microbenchmark
isolates model/planner scoring from environment simulation, rendering, video
generation, and logging.  All LeWM and RC-aux measurements use the horizon
$H=5$, 1024 candidate action sequences, 5 warmup calls, and 20 measured calls.
As shown in Table~\ref{tab:speed-microbenchmark}, reachability-aware scoring
adds less than $0.8$ ms per cost call over LeWM across all tasks.

DINO-WM-S/14 is included as a foundation-model-based baseline.  Its frozen
DINOv2-S/14~\citep{oquab2023dinov2} visual encoder substantially increases scoring-time computation.
On Push-T, DINO-WM-S/14 takes $4527.141$ ms per planner cost call with 512
candidate action sequences, while the 1024-candidate setting exceeds memory
under the same hardware setup.

\begin{table}[h!]
\centering

\begin{minipage}{0.5\textwidth}
\centering
\caption{
Controlled planner cost-call timing.
LeWM and RC-aux are measured with $H=5$ and 1024 candidate action
sequences. DINO-WM-S/14 is also included as a foundation-model-based baseline.
}
\label{tab:speed-microbenchmark}
\scalebox{0.65}{
\begin{tabular}{llrr}
\toprule
Task & Model & Candidates & ms / call \\
\midrule
\multirow{2}{*}{TwoRoom}
& LeWM~\citep{maes2026leworldmodel} & 1024 & 34.253 \\
& RC-aux & 1024 & 35.016 \\
\midrule
\multirow{2}{*}{Reacher}
& LeWM~\citep{maes2026leworldmodel} & 1024 & 34.893 \\
& RC-aux & 1024 & 35.361 \\
\midrule
\multirow{4}{*}{Push-T}

& DINO-WM-S/14~\citep{zhou2025dinowm} & 512 & 4527.141 \\
& DINO-WM-S/14~\citep{zhou2025dinowm} & 1024 & - \\
& LeWM~\citep{maes2026leworldmodel} & 1024 & 34.619 \\
& RC-aux & 1024 & 35.295 \\
\bottomrule
\end{tabular}
}
\end{minipage}
\hfill
\begin{minipage}{0.45\textwidth}
\centering
\caption{
LIBERO-Goal extension with an OFT-style action head.
The main matched comparison is between trainable LeWM and trainable RC-aux under the no-repeat protocol.
}
\label{tab:libero-goal}
\scalebox{0.62}{
\begin{tabular}{llc}
\toprule
Method & Setting & Mean success \\
\midrule
LeWM + OFT-style head & Trainable, no-repeat & 0.712 \\
RC-aux + OFT-style head & Trainable, no-repeat & 0.812 \\
RC-aux + OFT-style head & Repeat-tuned & 0.864 \\
\midrule
OpenVLA-OFT 7B~\citep{openvla1,openvla2} & External reference & 0.970 \\
\bottomrule
\end{tabular}
}
\end{minipage}

\end{table}

Overall, RC-aux retains the low-cost planning profile of LeWM while adding only
a small reachability module.  In contrast to DINO-WM, it does not rely on an
externally pretrained visual foundation-model encoder.


\subsection{LIBERO-Goal Extension}
\label{sec:libero}

We further evaluate whether the RC-aux representation transfers beyond the five latent-planning tasks.
For this experiment, we use LIBERO-Goal~\citep{libero} with an OFT-style~\citep{openvla1} action head trained on top of the learned representation.
This setting differs from the CEM-style latent planner used above: the representation is used by an action-prediction head rather than directly searched by a planner.
It therefore tests whether the reachability-corrected representation remains useful in a larger robot manipulation setting.

Table~\ref{tab:libero-goal} reports all-task mean success on LIBERO-Goal.
Under the aligned trainable no-repeat protocol, RC-aux improves over the trainable LeWM comparator from $0.712$ to $0.812$.
With repeat tuning on weaker tasks, RC-aux further reaches $0.864$ mean success.
This shows that the representation learned with RC-aux is not only useful for terminal latent planning but also provides a stronger initialization for downstream action-head training. Per-task LIBERO-Goal results are provided in Appendix~\ref{app:libero_details}.

\section{Conclusion} \label{sec:conclusion} We presented RC-aux, a lightweight auxiliary training objective for making reconstruction-free latent world models more useful for planning. Rather than modifying the LeWM backbone, RC-aux changes what the latent space is encouraged to represent: not only predictive consistency over future observations, but also whether a candidate state is attainable under a finite action budget. This distinction is crucial for goal-conditioned control, where latent proximity alone can make infeasible states appear deceptively promising to a planner. Our results suggest that a compact latent world model can become substantially more effective when its representation is trained around the operational questions asked at test time. In this sense, RC-aux is not simply an accuracy improvement, but a planning-alignment correction: it helps the model expose reachability structure that standard rollout losses tend to leave implicit. The current formulation still uses trajectory-derived reachability labels as a proxy for true environment-level attainability, and the test-time planner uses a simple reachability gate rather than a full decision-theoretic treatment of feasibility. Extending this idea to uncertainty-aware reachability, richer directed distances, and larger-scale manipulation settings is an important direction for future work.

\section{Limitations and Future Work}
\label{sec:limitations}

RC-aux learns finite-budget reachability from observed trajectory
offsets, which may overestimate the minimum number of steps needed
to reach a goal when shorter routes are absent from the data.
The learned signal therefore reflects dataset-supported reachability
rather than an exact shortest-path oracle. Aggregating evidence
across trajectories or using soft labels could improve this
supervision. The planner applies a lightweight soft correction
rather than enforcing feasibility constraints; uncertainty-aware
scoring and analysis under stochastic dynamics are natural
extensions. Our evaluation focuses on finite-budget planning,
with the full objective primarily tested on LeWorldModel.
The LIBERO-Goal extension provides complementary evidence of
representation transfer through an action-prediction head.
Further evaluation across backbones, data-coverage conditions,
and longer horizons would help characterize its broader
applicability.

\section*{Acknowledgments}

This work was partly supported by Council for Science, Technology and Innovation(CSTI), Cross-ministerial Strategic Innovation Promotion Program (SIP), the 3rd period of SIP "Creation of new ways of learning and working in a post-COVID-19 era" Grant Number JPJ012347 (Funding agency: JST), and JSPS KAKENHI Grant Numbers JP23K11211, JP24K02942, and JP25K21218.

\bibliographystyle{plain}
\bibliography{ref}


\newpage
\appendix

\section{Extended Related Work}
\label{app:extended_related}

\paragraph{Latent world models for control from pixels.}
World models learn compact predictive representations that support decision-making by imagining the consequences of actions.
This idea builds on earlier efforts to learn compact latent or predictive models for control from high-dimensional observations, including locally linear latent dynamics for raw-image control, generative world models, and action-conditioned visual foresight~\citep{watter2015embed,ha2018worldmodels,finn2017deepvisualforesight}.
Early pixel-based model-based agents such as PlaNet showed that latent dynamics can enable online planning from images~\cite{hafner2019planet}, while Dreamer-style agents learn behavior from imagined rollouts inside recurrent latent models~\citep{hafner2020dreamer,hafner2021dreamerv2,hafner2023dreamerv3}.
Control-centric approaches such as TD-MPC and TD-MPC2 further demonstrate that decoder-free latent models can be highly effective for continuous control when coupled with trajectory optimization or value learning~\cite{hansen2024tdmpc2}.
Related latent-planning agents also show that models need not reconstruct pixels to be useful for control: MuZero plans with a learned model that predicts planning-relevant quantities, while Plan2Explore uses self-supervised world models for exploration and fast downstream adaptation~\citep{schrittwieser2020muzero,sekar2020plan2explore}.
More recent reconstruction-free world models move closer to our setting.
Before DINO-WM and PLDM, several Dreamer-style variants already explored reducing or removing pixel reconstruction through temporal predictive coding, prototype-based objectives, contrastive objectives, or predictive task-relevant targets~\citep{nguyen2021temporal,deng2022dreamerpro,okada2022dreamingv2,burchi2024mudreamer}.
DINO-WM performs test-time goal planning by predicting future DINOv2 feature embeddings from offline trajectories~\citep{oquab2023dinov2,zhou2025dinowm}, while PLDM studies reward-free offline planning with latent dynamics learned through JEPA-style self-supervision~\cite{sobal2025pldm}.
LeWorldModel (LeWM) shows that an end-to-end JEPA world model can be trained stably from pixels using next-embedding prediction together with a Gaussian latent regularizer~\cite{maes2026leworldmodel}.
RC-aux builds directly on this line of reconstruction-free latent planning, but targets a different failure mode: even when a latent model is predictive, its learned geometry may not encode which goals are reachable within the planner's finite horizon.

\paragraph{Joint-embedding predictive architectures.}
Joint-Embedding Predictive Architectures (JEPAs) avoid pixel reconstruction by predicting future or masked representations in a learned latent space~\cite{lecun2022path}.
Image and video JEPAs, such as I-JEPA and V-JEPA, use masked prediction and target-encoder mechanisms to learn semantic visual representations~\cite{assran2023ijepa,bardes2023vjepa,assran2025vjepa2}.
These designs are related to a broader family of self-supervised representation-learning methods that use target networks, temporal prediction, or contrastive objectives to obtain useful visual features for downstream control~\citep{grill2020byol,schwarzer2020spr,laskin2020curl}.
For action-conditioned world modeling, the central challenge is to learn representations that do not collapse while remaining useful for control.
Some methods avoid collapse by freezing pretrained encoders~\cite{zhou2025dinowm}; others use multi-term regularizers inspired by VICReg~\cite{bardes2022vicreg,sobal2025pldm}.
LeWM instead uses SIGReg to enforce an isotropic Gaussian latent distribution and achieves stable end-to-end JEPA training from pixels~\cite{balestriero2025lejepa,maes2026leworldmodel}.
RC-aux keeps this backbone unchanged.
Rather than introducing a new anti-collapse mechanism, it adds planning-aligned supervision on top of the existing latent world model: multi-horizon open-loop prediction for temporal alignment and budget-conditioned reachability for geometric alignment.

\paragraph{Planning-aware representation geometry.}
A growing body of work shows that latent planning depends not only on prediction accuracy, but also on the geometry of the representation used by the planner.
Earlier work on plannable representations similarly learned latent spaces that support goal specification, trajectory optimization, or symbolic-style planning from high-dimensional observations~\citep{srinivas2018upn,kurutach2018causalinfogan}.
Value-guided JEPA shapes representation distances so that they approximate a goal-conditioned value function or quasi-distance~\citep{destrade2026valueguidedjepa}.
Temporal Straightening regularizes latent trajectories to reduce curvature, making Euclidean distance a better proxy for geodesic progress and improving the conditioning of latent planning~\citep{wang2026temporalstraightening}.
PcLast learns plannable continuous representations by associating reachable states in $\ell_2$ space~\citep{koul2024pclast}.
Complementary representation-learning work studies MDP-aware, bisimulation-based, or behavioral metrics that align latent distances with decision-relevant state similarity~\citep{gelada2019deepmdp,zhang2021invariant,castro2021mico}.
Related work in goal-conditioned reinforcement learning studies temporal distance, reachability, and asymmetric quasimetric structure as planning-relevant notions of state similarity~\citep{wang2023qrl,qian2023replan,bae2024tldr,steccanella2022state}.
RC-aux shares the view that planning requires a representation with operational geometry.
However, instead of replacing terminal latent planning with a learned value metric or imposing an unconditional geometric prior, RC-aux learns a \emph{budget-conditioned} reachability signal.
This distinction is important: a target may be eventually reachable but not reachable under the current planning budget, and this finite-horizon distinction is exactly what temporal hard negatives teach the model.

\paragraph{Multi-step prediction and planning mismatch.}
Multi-step prediction has long been used to reduce the mismatch between learned dynamics and planning-time rollouts, for example, through overshooting objectives in latent dynamics models~\citep{hafner2019planet} and self-supervised future-representation prediction objectives~\citep{schwarzer2020spr,nguyen2021temporal}.
Reconstruction-free variants further show that predictive objectives can improve robustness by avoiding unnecessary pixel-level detail~\citep{deng2022dreamerpro,burchi2024mudreamer}.
However, multi-step prediction alone does not determine whether the latent distance used by a goal-conditioned planner reflects finite-horizon reachability.
A model can roll out accurately and still assign low terminal distance to states separated by obstacles, irreversible dynamics, or manipulation bottlenecks.
RC-aux therefore combines multi-horizon open-loop prediction with budget-conditioned reachability.
The former aligns the learned dynamics with the open-loop rollouts used by the planner, while the latter shapes the latent space around horizon-indexed attainability.

\section{Planning-Alignment Analysis}
\label{app:planning_alignment}

This appendix formalizes the planning-alignment view of RC-aux.
The goal is not to prove a complete performance guarantee for high-dimensional visual control.
Instead, we isolate the two mismatches targeted by RC-aux:
(i) the mismatch between local prediction training and open-loop planning, and
(ii) the mismatch between Euclidean latent proximity and finite-budget reachability.
We show that multi-horizon open-loop prediction controls planning-time cost distortion, that trajectory-induced temporal hard negatives identify the budget dependence of reachability, and that the reachability gate acts as a soft feasibility bias over the base latent planning cost.

\subsection{Setup}

Consider a deterministic controlled system
\begin{equation}
    s_{t+1}=T(s_t,a_t),
    \qquad
    o_t=O(s_t),
\end{equation}
with observation map $O$.
A latent world model consists of an encoder $e_\theta$ and an action-conditioned latent predictor.
We write
\begin{equation}
    z_t=e_\theta(o_t).
\end{equation}
For an action sequence $\tau=a_{t:t+H-1}$, let
\begin{equation}
    \hat z_{t+1:t+H}
\end{equation}
denote the open-loop latent rollout predicted by the model, and let
\begin{equation}
    z_{t+1:t+H}^{\star}
\end{equation}
denote the encoded latent sequence obtained by executing the same action sequence in the true environment and encoding the resulting observations.
The base planner evaluates a latent goal cost
\begin{equation}
    C_{\mathrm{base}}(\hat z_{t+1:t+H},z_g),
\end{equation}
such as terminal latent distance, minimum distance over the rollout, soft-minimum distance, or mean distance.

\subsection{Open-loop prediction controls planning-time cost distortion}

The planner makes decisions using predicted rollouts.
Therefore, the relevant prediction error is not only the one-step error, but the error in the open-loop latent sequence on which the planner evaluates costs.

\begin{lemma}[Planning cost distortion from rollout error]
\label{lem:cost_distortion}
Assume the base planning cost $C_{\mathrm{base}}(\cdot,z_g)$ is $L_C$-Lipschitz with respect to the stacked rollout latent vector on the relevant latent region.
Then, for any candidate action sequence $\tau$,
\begin{equation}
\left|
C_{\mathrm{base}}(\hat z_{t+1:t+H},z_g)
-
C_{\mathrm{base}}(z_{t+1:t+H}^{\star},z_g)
\right|
\le
L_C
\left\|
\hat z_{t+1:t+H}
-
z_{t+1:t+H}^{\star}
\right\|_2 .
\end{equation}
\end{lemma}

\begin{proof}
This follows directly from Lipschitz continuity of $C_{\mathrm{base}}$ in its rollout argument.
\end{proof}

On a sampled training segment, $z_{t+k}$ is the encoded observation obtained after executing the sampled action sequence, and therefore corresponds to $z_{t+k}^{\star}$ for that segment.
For that segment, the RC-aux multi-horizon loss is
\begin{equation}
    \ell_{\mathrm{mh}}
    =
    \sum_{k=1}^{K}
    w_k
    \left\|
        \hat z_{t+k}
        -
        z_{t+k}
    \right\|_2^2.
\end{equation}
Assume $H\le K$ and define
\begin{equation}
    w_{\min}
    =
    \min_{1\le k\le H}w_k
    >
    0.
\end{equation}
Then
\begin{equation}
    \sum_{k=1}^{H}
    \left\|
        \hat z_{t+k}
        -
        z_{t+k}^{\star}
    \right\|_2^2
    \le
    \frac{1}{w_{\min}}
    \ell_{\mathrm{mh}}.
\end{equation}
Combining this inequality with Lemma~\ref{lem:cost_distortion} yields
\begin{equation}
\left|
C_{\mathrm{base}}(\hat z_{t+1:t+H},z_g)
-
C_{\mathrm{base}}(z_{t+1:t+H}^{\star},z_g)
\right|
\le
\frac{L_C}{\sqrt{w_{\min}}}
\sqrt{\ell_{\mathrm{mh}}}.
\end{equation}
Thus, multi-horizon prediction directly controls an upper bound on the distortion of the cost that the planner actually evaluates.

\paragraph{Terminal-cost special case.}
If the planner uses the terminal squared latent distance,
\begin{equation}
    C_{\mathrm{base}}(\hat z_{t+1:t+H},z_g)
    =
    \|\hat z_{t+H}-z_g\|_2^2,
\end{equation}
and all relevant latents satisfy $\|z\|_2\le B$, then the terminal cost is $4B$-Lipschitz on this bounded region:
\begin{equation}
\begin{aligned}
\left|
\|x-z_g\|_2^2
-
\|y-z_g\|_2^2
\right|
&=
\left|
\langle x-y,x+y-2z_g\rangle
\right|
\\
&\le
\|x-y\|_2
\left(
\|x\|_2+\|y\|_2+2\|z_g\|_2
\right)
\\
&\le
4B\|x-y\|_2.
\end{aligned}
\end{equation}

\subsection{Why one-step prediction is only an indirect control signal}

A one-step prediction objective can be locally accurate while still leaving planning-time rollout errors uncontrolled.
The following standard compounding argument makes this mismatch explicit.

\begin{lemma}[Compounding of one-step error]
\label{lem:compounding}
Let $F^\star$ be the true latent dynamics and $f_\theta$ the learned one-step predictor.
Assume $f_\theta$ is $L_f$-Lipschitz in the latent argument and the one-step prediction error is uniformly bounded on the rollout region:
\begin{equation}
    \|f_\theta(z,a)-F^\star(z,a)\|_2
    \le
    \epsilon_1.
\end{equation}
Then the $K$-step open-loop error satisfies
\begin{equation}
    \|\hat z_{t+K}-z_{t+K}^{\star}\|_2
    \le
    \epsilon_1
    \sum_{i=0}^{K-1}L_f^i.
\end{equation}
\end{lemma}

\begin{proof}
Let
\begin{equation}
    e_k
    =
    \|\hat z_{t+k}-z_{t+k}^{\star}\|_2.
\end{equation}
Then
\begin{equation}
\begin{aligned}
e_{k+1}
&=
\left\|
f_\theta(\hat z_{t+k},a_{t+k})
-
F^\star(z_{t+k}^{\star},a_{t+k})
\right\|_2
\\
&\le
\left\|
f_\theta(\hat z_{t+k},a_{t+k})
-
f_\theta(z_{t+k}^{\star},a_{t+k})
\right\|_2
+
\left\|
f_\theta(z_{t+k}^{\star},a_{t+k})
-
F^\star(z_{t+k}^{\star},a_{t+k})
\right\|_2
\\
&\le
L_f e_k+\epsilon_1.
\end{aligned}
\end{equation}
Since $e_0=0$, unrolling the recurrence gives the result.
\end{proof}

This lemma does not imply that one-step prediction is useless.
Rather, it shows that one-step prediction controls planning only indirectly, while RC-aux trains directly on the open-loop quantities queried by the planner.

\subsection{Finite-budget reachability as directed temporal geometry}

Let
\begin{equation}
    D^\star(s,g)
    =
    \inf_{\tau}
    \{h:T^h(s,\tau)=g\}
\end{equation}
be the shortest hitting time from $s$ to $g$, with $D^\star(s,g)=\infty$ if $g$ is unreachable.
In continuous environments, exact equality can be replaced by reaching an $\epsilon$-ball around the goal.
The ideal finite-budget reachability predicate is
\begin{equation}
    R_h^\star(s,g)
    =
    \mathbf 1[D^\star(s,g)\le h].
\end{equation}
This predicate is directed and budget-dependent.
In general,
\begin{equation}
    D^\star(s,g)\neq D^\star(g,s),
\end{equation}
especially in systems with obstacles, irreversible dynamics, or contact constraints.
Whenever the quantities are finite, $D^\star$ satisfies a directed triangle inequality:
\begin{equation}
    D^\star(s,u)
    \le
    D^\star(s,g)+D^\star(g,u),
\end{equation}
because a path from $s$ to $g$ can be concatenated with a path from $g$ to $u$.

Thus, finite-horizon planning is naturally described by horizon-indexed sublevel sets of a directed temporal distance:
\begin{equation}
    \{g:D^\star(s,g)\le h\}.
\end{equation}
By contrast, Euclidean latent distance is symmetric and budget-free.
RC-aux therefore learns a budget-conditioned predicate $R_\phi(z,z',h)$ rather than relying only on $\|z-z'\|_2$.

\subsection{Trajectory-induced reachability labels}

RC-aux does not assume access to $D^\star$.
Instead, it uses observed trajectory offsets as empirical supervision.
For a same-trajectory pair $(z_i,z_j)$ with $i<j$, define the observed offset
\begin{equation}
    \Delta=j-i.
\end{equation}
The trajectory-induced finite-budget label is
\begin{equation}
    y_{ijh}
    =
    \mathbf 1[h\ge\Delta].
\end{equation}

These labels should be interpreted carefully.
A positive label corresponds to a demonstrated feasible path within budget $h$ along the sampled trajectory.
A negative label with $h<\Delta$ should be interpreted as an empirical finite-budget negative under the observed trajectory distribution, rather than a certificate that the target is globally unreachable within $h$ under the true MDP.
A shorter unobserved path may exist, so the learned predicate is best understood as dataset-induced reachability unless additional coverage assumptions hold.

Let $D_{\mathcal D}(s,g)$ denote the shortest observed offset from $s$ to $g$ in the dataset, when such an offset exists.
Then any observed segment is a feasible path, so
\begin{equation}
    D^\star(s,g)
    \le
    D_{\mathcal D}(s,g).
\end{equation}
If the behavior data is $c$-competitive in the sense that
\begin{equation}
    D_{\mathcal D}(s,g)
    \le
    c\,D^\star(s,g),
\end{equation}
then dataset reachability approximates true reachability up to a multiplicative slack.
Specifically,
\begin{equation}
    D_{\mathcal D}(s,g)\le h
    \Rightarrow
    D^\star(s,g)\le h,
\end{equation}
while
\begin{equation}
    D_{\mathcal D}(s,g)>h
    \Rightarrow
    D^\star(s,g)>h/c.
\end{equation}
This formalizes the role of data coverage: trajectory-induced labels become closer to true finite-budget reachability as the offline trajectories become more complete and less suboptimal.

\subsection{Bayes-optimal reachability under binary supervision}

For a latent pair and budget, let $Y\in\{0,1\}$ denote the trajectory-induced finite-budget label.
The reachability head is trained with binary cross-entropy:
\begin{equation}
    \mathcal L_{\mathrm{bce}}
    =
    \mathbb E
    \left[
        -Y\log R_\phi(Z,Z',h)
        -
        (1-Y)\log(1-R_\phi(Z,Z',h))
    \right].
\end{equation}

\begin{proposition}[Bayes-optimal reachability predictor]
\label{prop:bayes_reach}
For any fixed representation, the minimizer of the binary cross-entropy risk is
\begin{equation}
    R^\star(z,z',h)
    =
    \mathbb P(Y=1\mid Z=z,Z'=z',h).
\end{equation}
If the representation preserves the relevant temporal relation and the labels are deterministic, then for same-trajectory pairs with offset $\Delta=j-i$,
\begin{equation}
    R^\star(z_i,z_j,h)
    =
    \mathbf 1[h\ge\Delta].
\end{equation}
\end{proposition}

\begin{proof}
For a fixed input $(z,z',h)$, let $p=\mathbb P(Y=1\mid z,z',h)$ and let $q$ be the predicted probability.
The conditional BCE risk is $-p\log q-(1-p)\log(1-q)$, whose pointwise minimizer is $q=p$.
The deterministic-label case follows immediately.
\end{proof}

This proposition clarifies what the reachability head estimates: a conditional probability of empirical finite-budget reachability under the data-generating distribution.

\subsection{Trajectory-induced temporal hard negatives identify the budget}

The most important role of trajectory-induced temporal hard negatives is to make the budget variable identifiable.

\begin{proposition}[Budget identifiability]
\label{prop:budget_identifiability}
Consider same-trajectory latent pairs $(z_i,z_j)$ with temporal offset $\Delta=j-i$.
If reachability supervision includes same-trajectory positives and batch negatives, but never includes same-trajectory pairs with insufficient budget, then there exists a classifier that fits all observed labels while ignoring the budget $h$ on same-trajectory pairs.
If trajectory-induced temporal hard negatives with $h<\Delta$ are included, then any classifier that fits both labels for the same ordered pair must depend on $h$.
\end{proposition}

\begin{proof}
Without trajectory-induced temporal hard negatives, every observed same-trajectory pair used by the reachability loss is labeled positive, while batch negatives are labeled negative.
A classifier can therefore predict positive for same-trajectory pairs and negative for cross-trajectory pairs, independent of $h$.
This fits the labels without learning finite-budget reachability.

With trajectory-induced temporal hard negatives, the same ordered pair $(z_i,z_j)$ can be labeled negative for budgets $h<\Delta$ and positive for budgets $h\ge\Delta$.
Any classifier that fits both labels must assign different predictions to the same latent pair under different budgets.
Therefore, it cannot ignore $h$.
\end{proof}

This result explains why hard negatives are not merely additional negative samples.
They are the samples that identify the horizon dependence of the reachability predicate.

\subsection{Reachability-aware planning as a soft feasibility bias}

RC-aux optionally modulates the base latent planning cost with the trajectory-level reachability score
\begin{equation}
    R(\tau,z_g)
    =
    \max_{1\le k < H}
    R_\phi(\hat z_{t+k},z_g,H-k).
\end{equation}
The planning cost is
\begin{equation}
    C_{\mathrm{RC}}(\tau)
    =
    C_{\mathrm{base}}(\tau)
    \cdot
    \max(m,1-\lambda_{\mathrm{plan}}R(\tau,z_g)),
\end{equation}
where $m>0$ is a small floor.

Ignoring the floor $m$ for clarity, and assuming $0\le\lambda_{\mathrm{plan}}<1$, consider two trajectories $\tau_a$ and $\tau_b$ with base costs
\begin{equation}
    d_a=C_{\mathrm{base}}(\tau_a),
    \qquad
    d_b=C_{\mathrm{base}}(\tau_b),
\end{equation}
and reachability scores
\begin{equation}
    R_a=R(\tau_a,z_g),
    \qquad
    R_b=R(\tau_b,z_g).
\end{equation}
RC-aux prefers $\tau_a$ over $\tau_b$ when
\begin{equation}
    d_a(1-\lambda_{\mathrm{plan}}R_a)
    <
    d_b(1-\lambda_{\mathrm{plan}}R_b).
\end{equation}
Equivalently,
\begin{equation}
    \frac{d_a}{d_b}
    <
    \frac{
        1-\lambda_{\mathrm{plan}}R_b
    }{
        1-\lambda_{\mathrm{plan}}R_a
    }.
\end{equation}
If $\tau_a$ has high reachability $R_a\approx 1$ and $\tau_b$ has low reachability $R_b\approx 0$, then $\tau_a$ can be selected even if its base latent cost is larger, as long as
\begin{equation}
    \frac{d_a}{d_b}
    <
    \frac{1}{1-\lambda_{\mathrm{plan}}}.
\end{equation}
Thus, reachability-aware planning can prefer a slightly longer but empirically attainable trajectory over a deceptively short latent-space shortcut.
When $\lambda_{\mathrm{plan}}=0$, the condition reduces to standard base-cost planning.

\subsection{Predicted rollout pairs reduce planner-time distribution shift}

The reachability head is queried during planning on predicted latents $\hat z$, not only on encoded dataset latents $z$.
Training reachability only on encoded latents may leave the classifier poorly calibrated on the planner-induced rollout distribution.
Predicted rollout pairs address this mismatch.

Let $r_\phi(z,z',h)$ be the reachability logit, and suppose that on encoded pairs it classifies with margin $\gamma>0$:
\begin{equation}
    y\,r_\phi(z,z',h)\ge\gamma,
    \qquad
    y\in\{-1,+1\}.
\end{equation}
Assume $r_\phi$ is $L_r$-Lipschitz in its source latent.
If a predicted latent satisfies
\begin{equation}
    \|\tilde z-z\|_2\le\delta
\end{equation}
and
\begin{equation}
    L_r\delta<\gamma,
\end{equation}
then replacing $z$ by $\tilde z$ does not change the sign of the reachability prediction:
\begin{equation}
    y\,r_\phi(\tilde z,z',h)>0.
\end{equation}
Thus, reliable reachability on planner-time rollouts requires the classifier to be robust on predicted latents.
RC-aux trains this regime directly by applying reachability supervision to predicted rollout pairs.
If stop-gradient is used on predicted latents, these terms calibrate the reachability head on planner-induced latents without directly backpropagating reachability labels into the rollout model.
Without stop-gradient, they additionally propagate the reachability signal into the predicted rollout states.

\subsection{Summary}

The analysis supports the design of RC-aux in three ways.
First, multi-horizon open-loop prediction controls the distortion of the latent costs evaluated by the planner.
Second, budget-conditioned reachability learns empirical horizon-indexed reachable sets, and trajectory-induced temporal hard negatives are necessary to make the budget dependence identifiable.
Third, the optional reachability-aware planner acts as a soft feasibility bias over the base latent cost, recovering the base planner when $\lambda_{\mathrm{plan}}=0$ and favoring empirically attainable rollouts when $\lambda_{\mathrm{plan}}>0$.

\section{Experimental Details}
\label{app:exp_details}

This appendix provides additional details for the experimental protocol used in Section~\ref{sec:experiments}, including task setup, baseline definitions, planner hyperparameters, and LIBERO-Goal training details.

\paragraph{Task suite.}
The main evaluation uses five pixel-based goal-conditioned control tasks: TwoRoom, Reacher, Push-T, Wall, and Cube.
Each episode provides an initial observation and a goal image.
The agent executes actions from the initial observation, and success is determined by whether the resulting trajectory reaches the specified goal within the task horizon.
TwoRoom evaluates navigation in a top-down room environment.
Reacher evaluates goal-conditioned reaching.
Push-T evaluates object pushing from image observations.
Wall evaluates obstacle-constrained planning, where Euclidean closeness in latent space can be misleading if the goal is not reachable through a direct path.
Cube evaluates goal-conditioned 3D manipulation.

\paragraph{Evaluation protocol.}
We report the success rate as the primary metric.
For local LeWM-family runs, each method is evaluated on five fixed evaluation groups with 50 episodes per group, and we report the mean and group-level standard deviation.
The same fixed groups are used across local LeWM-family comparisons, enabling paired episode-level analyses in Appendix~\ref{app:additional_results}.
Previously reported baseline results are included as benchmark references in the main comparison table, while local LeWM-family rows are reported with group-level standard deviations.

\paragraph{Baselines.}
We compare RC-aux with prior baselines and LeWM-family controls.
The prior baselines include reconstruction-free visual world models and latent-dynamics methods such as DINO-WM~\citep{zhou2025dinowm} and PLDM~\citep{two}, goal-conditioned imitation learning such as GCBC~\citep{GCBC}, and offline value-learning methods such as IQL~\citep{IQL} and IVL~\citep{ogbench}.
When available, we also include task-specific references such as DINO-WM+prop on Push-T.
The LeWM-family rows provide the closest controlled comparison: LeWM is the original backbone baseline, LeWM-cont is a continuation-trained control using the same backbone, and RC-aux augments LeWM~\citep{maes2026leworldmodel} with the proposed reachability-correction objective and reachability-aware planner.

\paragraph{LeWM-family controls.}
For all five tasks, LeWM-cont serves as the primary comparator,
accounting for additional training under the same backbone.
The matched $\Delta$ row in Table~\ref{tab:main-success}
is computed relative to LeWM-cont.

\paragraph{Planner and evaluation hyperparameters.}
Table~\ref{tab:planner-eval-hparams} lists the main planner and evaluation hyperparameters.
All local LeWM-family results use the same fixed evaluation grouping described above.
For Wall, we use the tuned planner horizon configuration used for that environment.
For Cube, the planner reachability weight is set to zero in the reported configuration.

\begin{table*}[t]
\centering
\caption{
Planner and evaluation hyperparameters for the five-task pixel-control evaluation.
Each local LeWM-family method is evaluated with five fixed groups of 50 episodes.
}
\label{tab:planner-eval-hparams}
\begin{tabular}{lcccccc}
\toprule
Task & CEM samples & CEM iters & Top-$k$ & Budget & Goal offset / horizon & $\lambda_{\mathrm{plan}}$ \\
\midrule
TwoRoom & 300 & 30 & 30 & 50 & 25 & 0.35 \\
Reacher & 300 & 30 & 30 & 50 & 25 & 0.35 \\
Push-T & 300 & 30 & 30 & 50 & 25 & 0.35 \\
Wall & 600 & 20 & 60 & 60 & $h=8,\ rh=4$ & 0.85 \\
Cube & 300 & 10 & 30 & 50 & 25 & 0.00 \\
\bottomrule
\end{tabular}
\end{table*}

\paragraph{LIBERO-Goal OFT setup.}
For the LIBERO-Goal extension, we evaluate all 10 LIBERO-Goal tasks using the official success checker with 50 evaluation episodes per task and a maximum horizon of 600 environment steps.
The policy uses two RGB views, \texttt{agentview\_rgb} and \texttt{eye\_in\_hand\_rgb}, and trains an OFT-style action chunk head on top of the learned representation.
The main matched comparison uses trainable encoders and the no-repeat all-task setting for both LeWM and RC-aux.
We additionally report a repeat-tuned RC-aux result as the best observed RC-aux setting.

\begin{table}[t]
\centering
\caption{
Training and evaluation settings for the LIBERO-Goal OFT-style extension.
}
\label{tab:libero-oft-hparams}
\begin{tabular}{lc}
\toprule
Setting & Value \\
\midrule
Tasks & LIBERO-Goal, tasks 0--9 \\
Evaluation episodes & 50 per task \\
Maximum steps & 600 \\
Observation views & \texttt{agentview\_rgb}, \texttt{eye\_in\_hand\_rgb} \\
Action head & OFT-style chunk head \\
Chunk length & 8 \\
Action horizon & 8 \\
Batch size & 32 \\
Learning rate & $5\times 10^{-5}$ \\
Weight decay & $1\times 10^{-4}$ \\
Precision & bfloat16 \\
Loss & $\ell_1$ action loss \\
Main comparator & Trainable, no-repeat \\
\bottomrule
\end{tabular}
\end{table}

\section{LIBERO-Goal Extension Details}
\label{app:libero_details}

This appendix reports per-task LIBERO-Goal results for the OFT-style action-head extension.
The training and evaluation setup is described in Appendix~\ref{app:exp_details}.
The mean-level results are summarized in Section~\ref{sec:libero}; here we provide the task-level breakdown.
Table~\ref{tab:libero-goal-pertask} reports per-task success rates.
Under the matched no-repeat protocol, RC-aux improves or matches LeWM on all 10 tasks, with the largest gains on tasks 0, 8, and 9.
The repeat-tuned setting further improves weaker tasks, especially task 5.

\begin{table*}[h]
\centering
\caption{
Per-task LIBERO-Goal success rates.
The matched comparison is between LeWM trainable no-repeat and RC-aux trainable no-repeat.
Repeat-tuned RC-aux and OpenVLA-OFT 7B~\citep{openvla1,openvla2} are included as additional references.
}
\label{tab:libero-goal-pertask}
\scalebox{0.8}{\begin{tabular}{lccccccccccc}
\toprule
Method & T0 & T1 & T2 & T3 & T4 & T5 & T6 & T7 & T8 & T9 & Mean \\
\midrule
LeWM, trainable no-repeat
& 0.64 & 0.78 & 0.70 & 0.76 & 0.90 & 0.44 & 0.60 & 0.94 & 0.70 & 0.66 & 0.712 \\
RC-aux, trainable no-repeat
& 0.92 & 0.86 & 0.80 & 0.78 & 0.96 & 0.48 & 0.70 & 0.96 & 0.86 & 0.80 & 0.812 \\
RC-aux, repeat-tuned
& 0.94 & 0.86 & 0.76 & 0.88 & 0.92 & 0.84 & 0.74 & 1.00 & 0.88 & 0.82 & 0.864 \\
OpenVLA-OFT 7B
& 0.98 & 0.92 & 0.96 & 0.86 & 1.00 & 1.00 & 1.00 & 1.00 & 1.00 & 0.98 & 0.970 \\
\bottomrule
\end{tabular}}
\end{table*}

\section{Additional Quantitative Results}
\label{app:additional_results}

This appendix provides additional quantitative evidence for the results in Section~\ref{sec:experiments}.
We include paired fixed-episode outcomes and local
LeWM-family success visualizations.

\subsection{Paired Fixed-Episode Outcomes}
\label{app:paired_outcomes}

To verify that the aggregate gains are not caused by differences in evaluation episodes, we compare LeWM and RC-aux on the same fixed evaluation episodes.
Table~\ref{tab:paired-outcomes} reports episode-level outcome counts.
RC-aux solves many episodes missed by LeWM, especially on Wall and TwoRoom.
On Wall, RC-aux-only successes outnumber LeWM-only successes by $85$ to $2$.
Push-T is the most balanced task, consistent with the main results.

\begin{table}[h]
\centering
\caption{
Paired fixed-episode outcomes for LeWM and RC-aux.
Each entry counts evaluation episodes in one of four categories: both methods fail, only LeWM succeeds, only RC-aux succeeds, or both methods succeed.
}
\label{tab:paired-outcomes}
\begin{tabular}{lrrrr}
\toprule
Task & Both fail & LeWM only & RC-aux only & Both succeed \\
\midrule
TwoRoom & 4 & 1 & 24 & 221 \\
Reacher & 15 & 17 & 32 & 186 \\
Push-T & 10 & 13 & 14 & 213 \\
Wall & 39 & 2 & 85 & 124 \\
Cube & 54 & 6 & 15 & 175 \\
\bottomrule
\end{tabular}
\end{table}

Figure~\ref{fig:paired-outcomes} visualizes the same paired outcomes over the fixed evaluation groups.
The visual pattern highlights that the largest aggregate gains correspond to many RC-aux-only successes on the same evaluation episodes.

\begin{figure*}[h]
    \centering
    \includegraphics[width=0.98\linewidth]{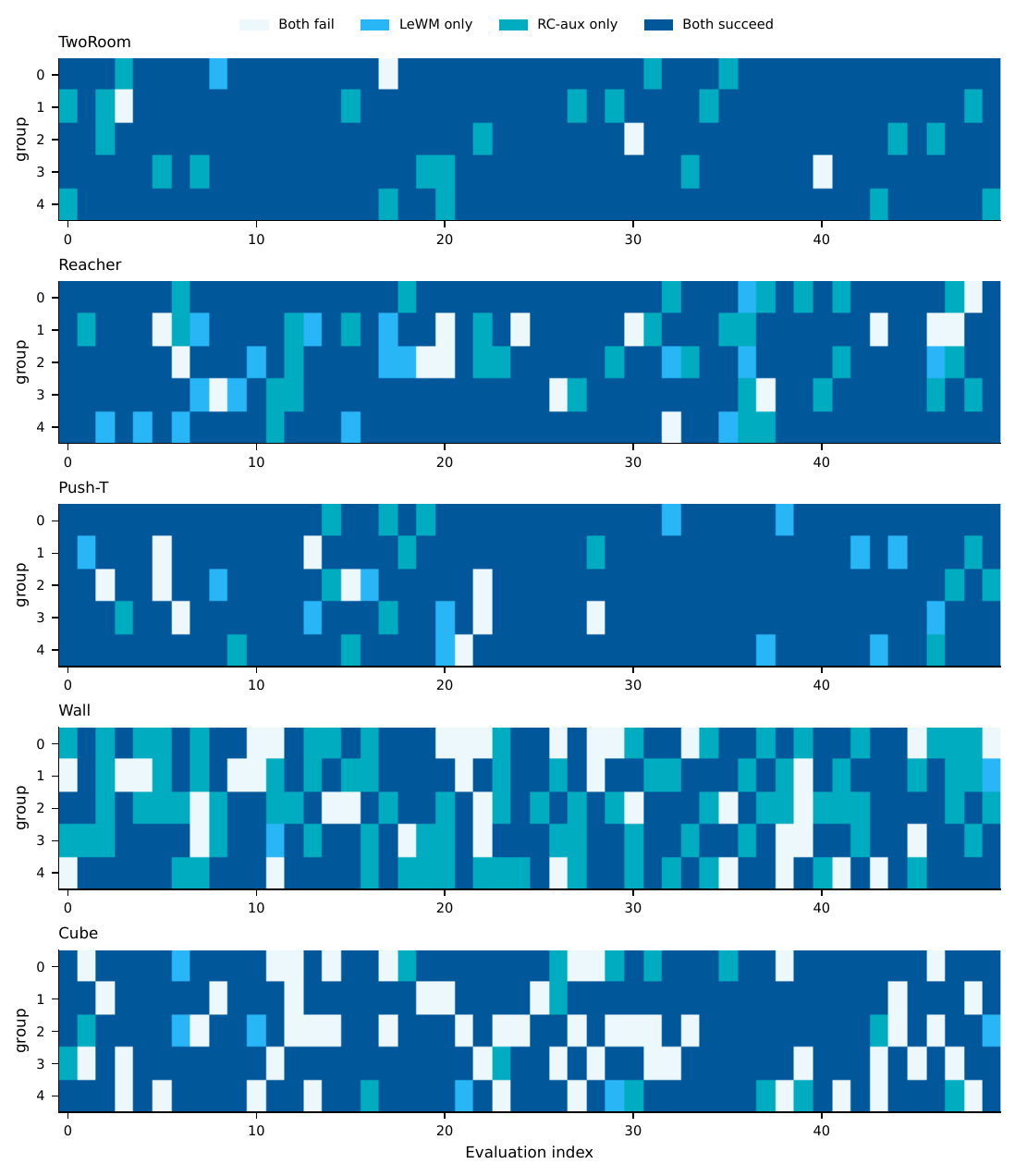}
    \caption{
    Paired outcomes on fixed evaluation episodes.
    Each cell corresponds to one evaluation episode.
    RC-aux-only successes are most frequent on Wall and TwoRoom, while Push-T remains nearly balanced between LeWM-only and RC-aux-only successes.
    }
    \label{fig:paired-outcomes}
\end{figure*}

\subsection{Local LeWM-Family Success Visualizations}
\label{app:local_success}

Figure~\ref{fig:local-main-success} shows the local LeWM-family success rates with group-level standard deviations.
This visualization isolates the comparison among LeWM, continuation-trained LeWM, and RC-aux, complementing the broader baseline comparison in the main text.
For all five tasks, the matched comparator is LeWM-cont,
consistent with the matched deltas reported in
Table~\ref{tab:main-success}.

\begin{figure}[h]
    \centering
    \includegraphics[width=0.8\linewidth]{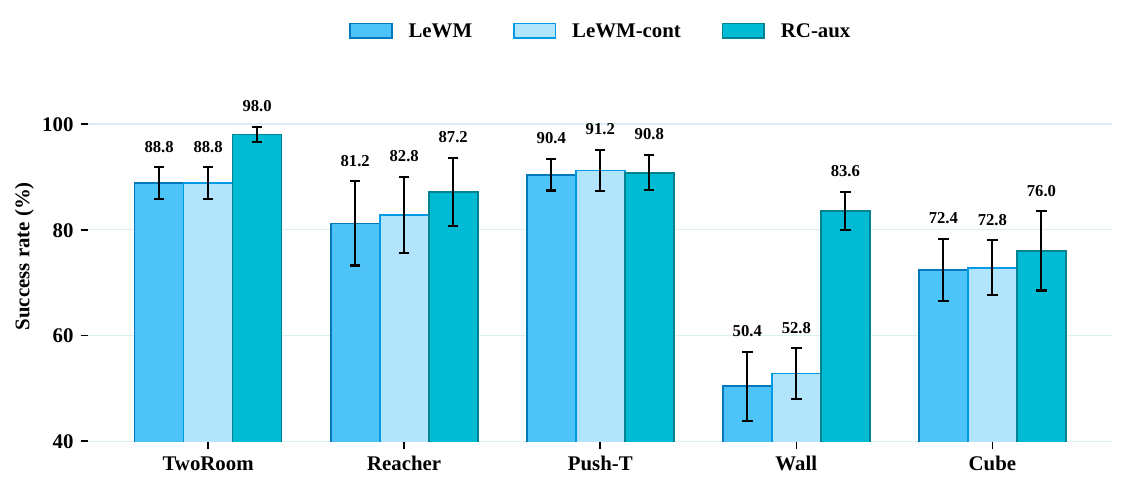}
    \caption{
    Local LeWM-family success rates.
    Bars show mean success over five fixed evaluation groups, and error bars show group-level standard deviation.
   For matched controlled comparisons, RC-aux is compared
against LeWM-cont across all five tasks.
    }
    \label{fig:local-main-success}
\end{figure}

\section{Qualitative Results}
\label{app:additional_qualitative}

\subsection{Paired Wall and Cube Comparisons}
\label{app:paired_qualitative}

Figure~\ref{fig:wall-qualitative-app} shows paired Wall rollouts.
Wall provides a visually clear example of obstacle-constrained reachability: latent proximity to the goal can be misleading when the direct path is blocked.
RC-aux produces rollouts that more consistently follow attainable paths toward the goal.

\begin{figure*}[h]
    \centering
    \includegraphics[width=0.98\linewidth]{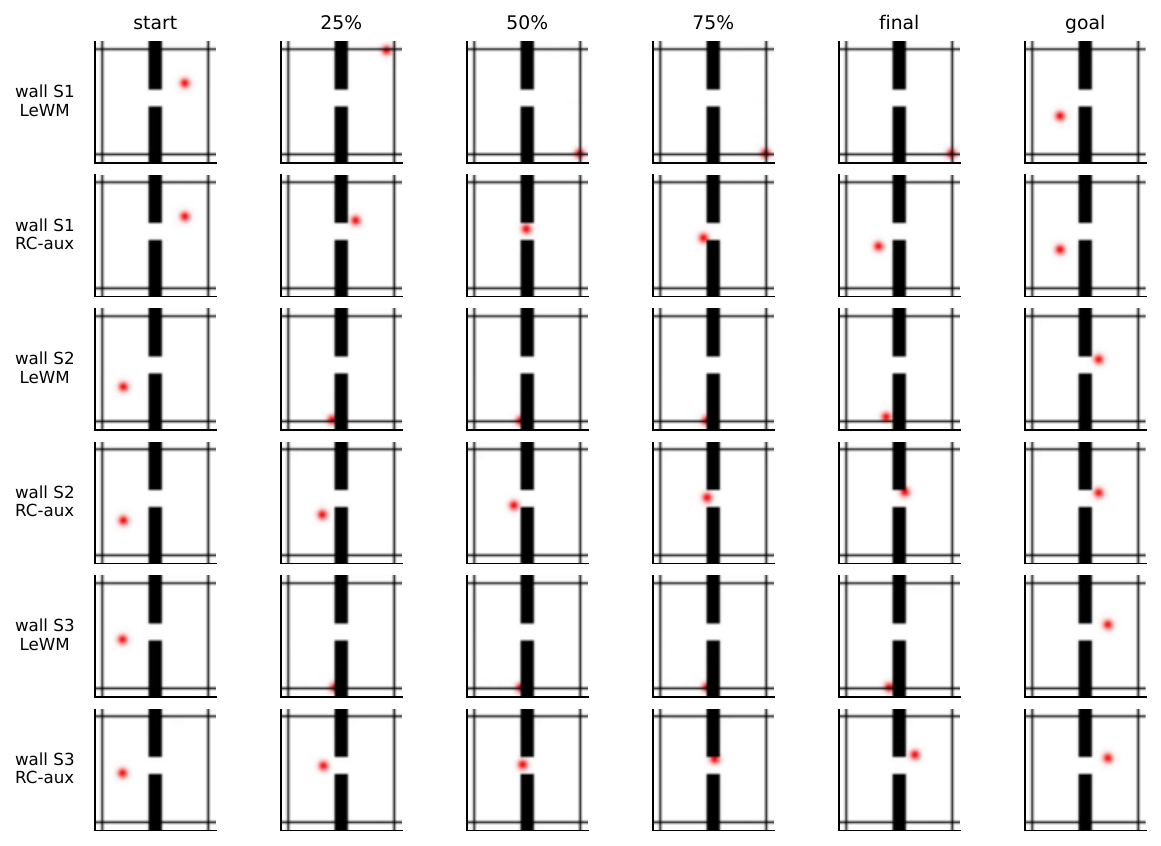}
    \caption{
    Paired Wall rollouts comparing LeWM and RC-aux.
    Wall highlights obstacle-constrained reachability, where visually or latently nearby states may not be attainable within the planning budget.
    }
    \label{fig:wall-qualitative-app}
\end{figure*}

Figure~\ref{fig:cube-qualitative-app} provides the full Cube comparison.
This figure complements the main-text qualitative result by showing additional Cube rollouts under the same paired comparison format.

\begin{figure*}[h]
    \centering
    \includegraphics[width=0.98\linewidth]{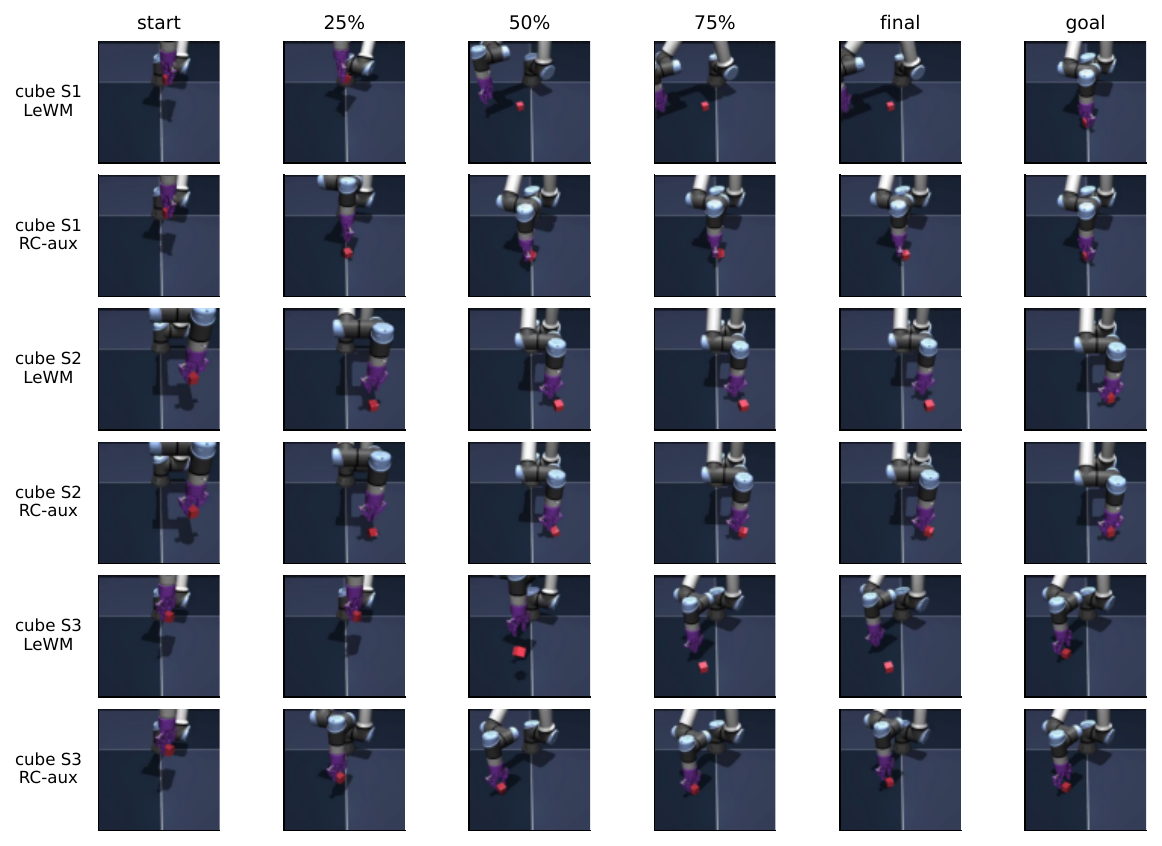}
    \caption{
    Full paired Cube rollouts comparing LeWM and RC-aux.
    The figure provides additional examples beyond the compact Cube visualization in the main text.
    }
    \label{fig:cube-qualitative-app}
\end{figure*}

\subsection{Additional RC-aux Success Rollouts}
\label{app:rcaux_success_rollouts}

Figures~\ref{fig:rcaux-success-wall}--\ref{fig:rcaux-success-cube} show additional successful RC-aux rollouts across the five tasks.
These examples illustrate the visual diversity of the benchmark and the types of goal-conditioned behavior produced by the RC-aux latent planner.

\begin{figure*}[t]
    \centering
    \includegraphics[width=0.98\linewidth]{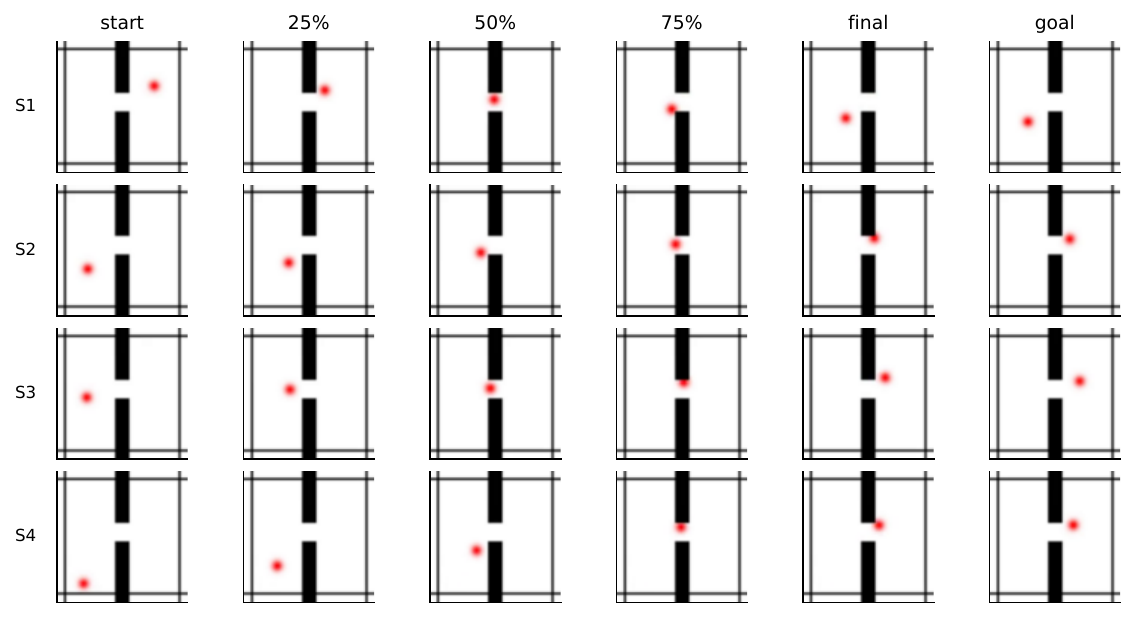}
    \caption{
    Additional successful RC-aux rollouts on Wall.
    }
    \label{fig:rcaux-success-wall}
\end{figure*}

\begin{figure*}[t]
    \centering
    \includegraphics[width=0.98\linewidth]{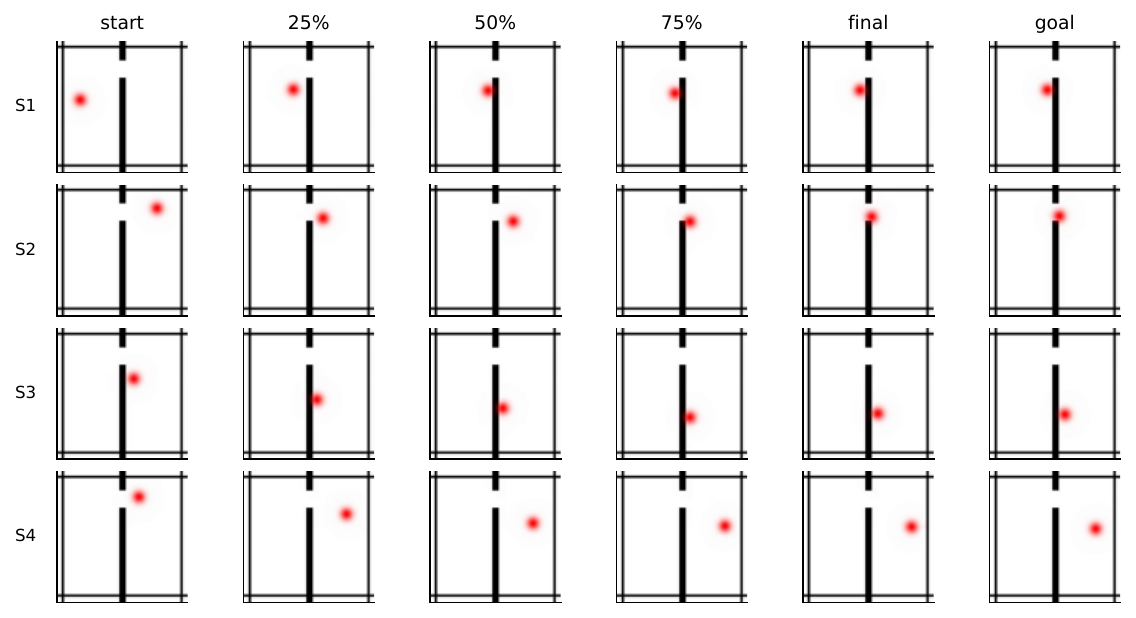}
    \caption{
    Additional successful RC-aux rollouts on TwoRoom.
    }
    \label{fig:rcaux-success-tworoom}
\end{figure*}

\begin{figure*}[t]
    \centering
    \includegraphics[width=0.98\linewidth]{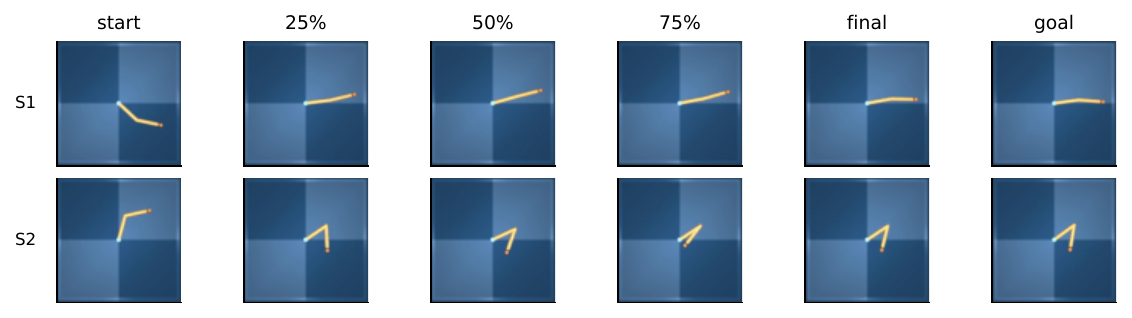}
    \caption{
    Additional successful RC-aux rollouts on Reacher.
    }
    \label{fig:rcaux-success-reacher}
\end{figure*}

\begin{figure*}[t]
    \centering
    \includegraphics[width=0.98\linewidth]{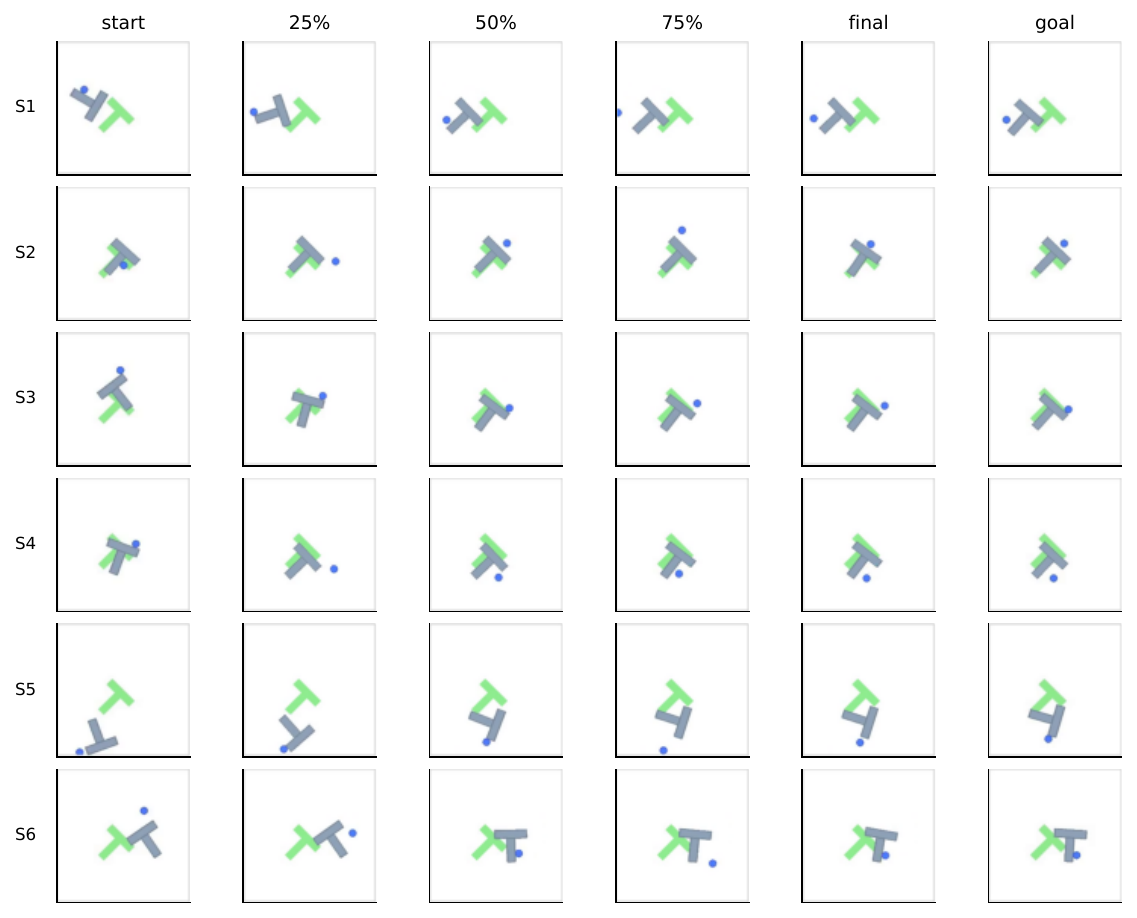}
    \caption{
    Additional successful RC-aux rollouts on Push-T.
    }
    \label{fig:rcaux-success-pusht}
\end{figure*}

\begin{figure*}[t]
    \centering
    \includegraphics[width=0.98\linewidth]{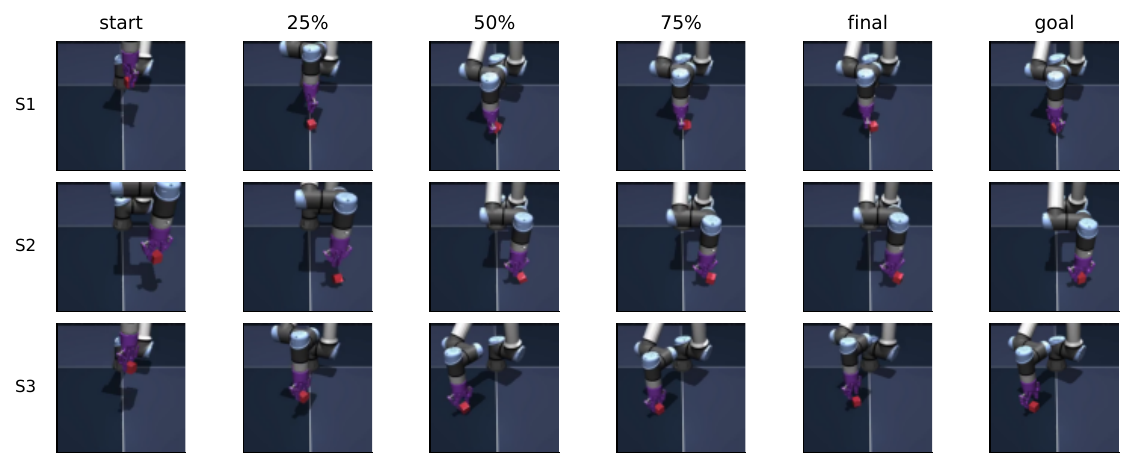}
    \caption{
    Additional successful RC-aux rollouts on Cube.
    }
    \label{fig:rcaux-success-cube}
\end{figure*}

\subsection{Approximate Pixel-Space Trajectory Overlays}
\label{app:trajectory_overlays}

Figure~\ref{fig:trajectory-overlays} provides approximate pixel-space trajectory overlays for Wall and TwoRoom.
The trajectories are extracted from rendered videos by tracking the visible object or agent center, and should therefore be interpreted as qualitative visualizations rather than simulator ground-truth state trajectories.

\begin{figure*}[h]
    \centering
    \includegraphics[width=0.85\linewidth]{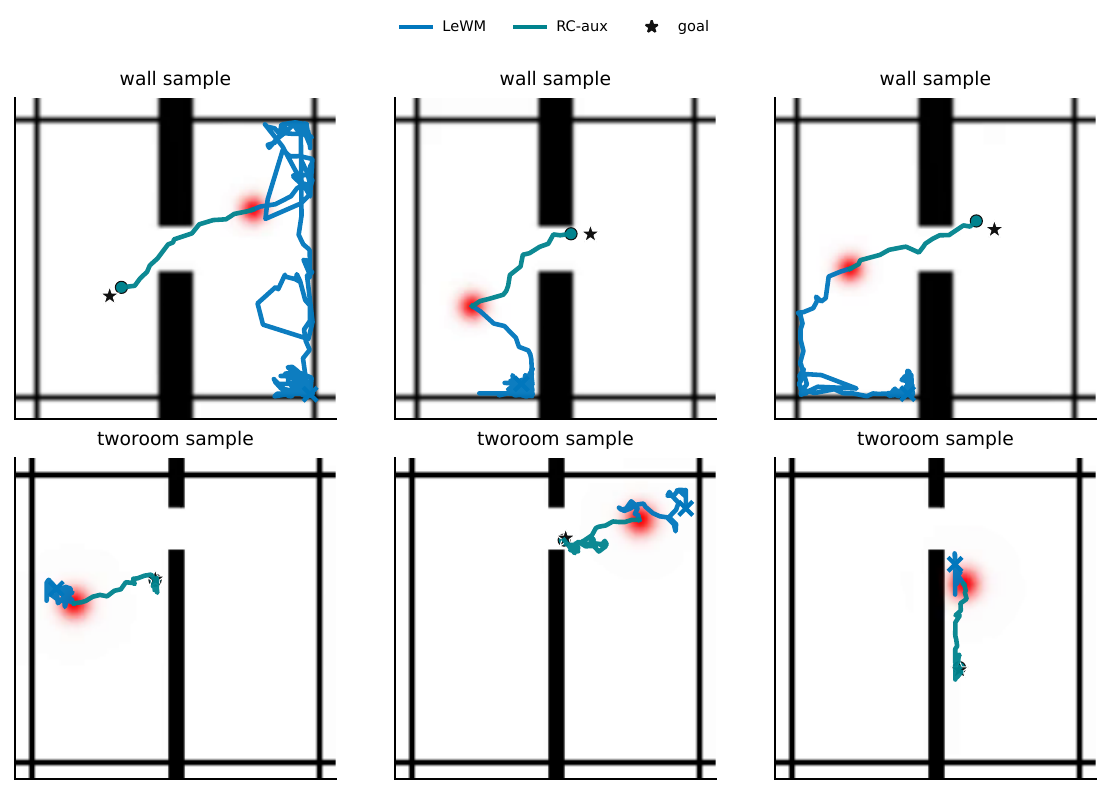}
    \caption{
    Approximate pixel-space trajectory overlays.
    Paths are extracted from rendered rollout videos by color tracking.
    The visualization is qualitative and is intended to illustrate the trajectory-level behavior of LeWM and RC-aux.
    }
    \label{fig:trajectory-overlays}
\end{figure*}

\section{Latent and Physical Diagnostics}
\label{app:diagnostics}

This appendix provides additional diagnostics of the learned latent representation.
These analyses complement the success-rate results by examining latent terminal distances, reachability scores, and physical probing on Push-T.

\subsection{Latent Reachability Diagnostic}
\label{app:latent_diagnostic}

Figure~\ref{fig:latent-reachability-diagnostic} visualizes selected Wall and TwoRoom rollouts in a two-dimensional PCA projection of the corresponding model latents.
The figure is diagnostic rather than a matched quantitative benchmark: each panel uses the corresponding model encoder and PCA projection.
The high-dimensional terminal latent distances reported in the titles provide a more reliable comparison.

\begin{figure*}[t]
    \centering
    \includegraphics[width=0.98\linewidth]{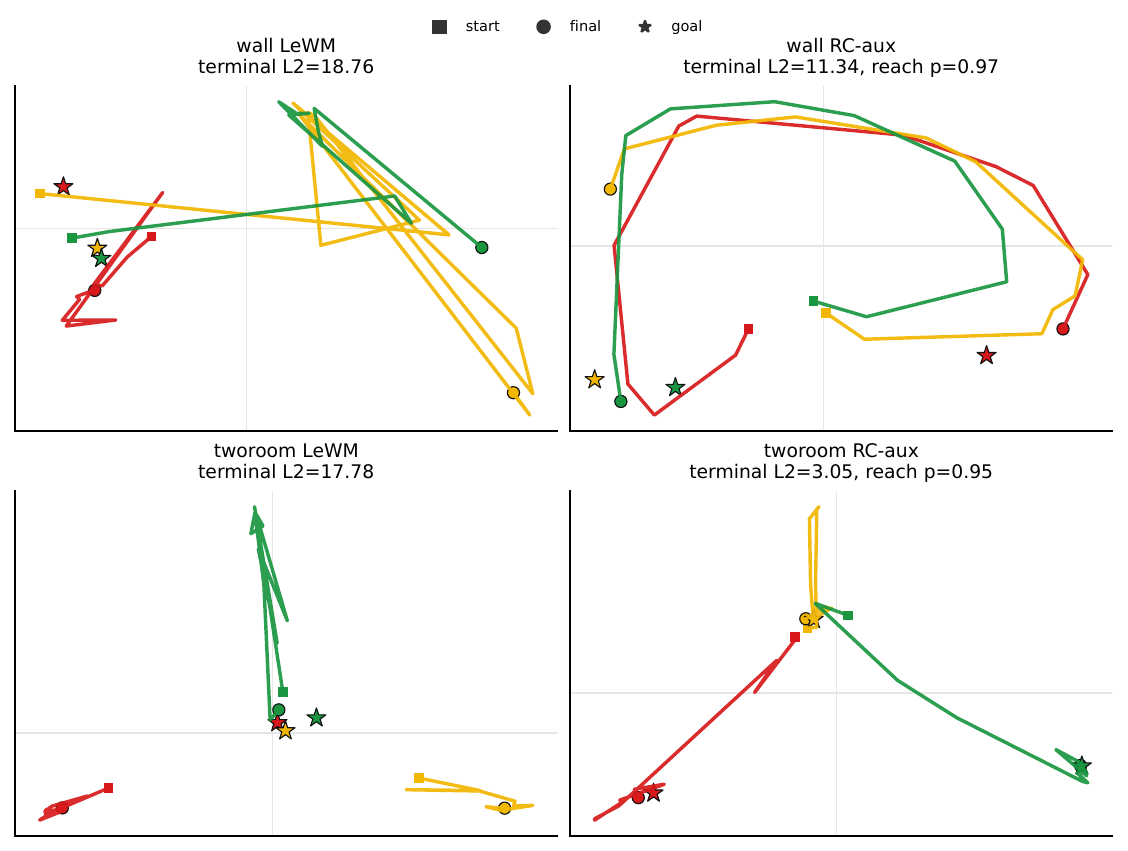}
    \caption{
    Latent reachability diagnostic on selected Wall and TwoRoom rollouts.
    Rollout frames and goals are encoded by the corresponding model and projected to two dimensions by PCA for visualization.
    Squares, circles, and stars denote start, final, and goal latents.
    The diagnostic suggests that RC-aux terminal latents are closer to the encoded goals and are assigned high reachability by the learned reachability head.
    }
    \label{fig:latent-reachability-diagnostic}
\end{figure*}

Figure~\ref{fig:terminal-distance-summary} summarizes terminal latent distances on the selected diagnostic examples.
The RC-aux rollouts end closer to the goal latents in these examples, consistent with the paired-outcome and qualitative rollout results.

\begin{figure}[t]
    \centering
    \includegraphics[width=0.3\linewidth]{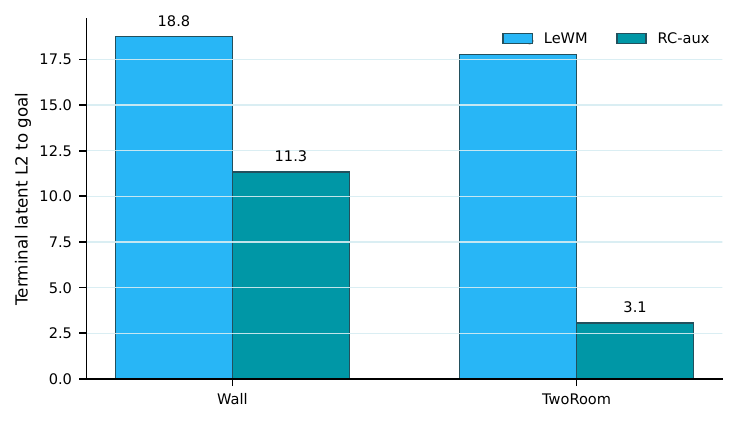}
    \caption{
    Terminal latent distance summary for selected diagnostic rollouts.
    Lower distance indicates that the final rollout latent is closer to the encoded goal latent.
    }
    \label{fig:terminal-distance-summary}
\end{figure}

\subsection{Push-T Physical Probe}
\label{app:physical_probe}

We further evaluate whether RC-aux preserves physical information in the latent representation on Push-T.
Table~\ref{tab:physical-probe-summary} summarizes local probe results for representative physical quantities.
RC-aux is comparable to LeWM across these probes, suggesting that the reachability correction does not discard basic physical information.

\begin{table}[t]
\centering
\caption{
Push-T physical probe summary.
Lower MSE is better.
The results indicate that RC-aux preserves physical information in the latent representation.
}
\label{tab:physical-probe-summary}
\begin{tabular}{llcc}
\toprule
Target & Probe metric & LeWM & RC-aux \\
\midrule
Agent location & Linear MSE & 0.039 & 0.037 \\
Block location & Linear MSE & 0.021 & 0.020 \\
Block angle & Linear MSE & 0.170 & 0.157 \\
Block angle & MLP MSE & 0.017 & 0.020 \\
\bottomrule
\end{tabular}
\end{table}

Figure~\ref{fig:physical-probe-pusht} visualizes the Push-T linear-probe MSE comparison.

\begin{figure}[h]
    \centering
    \includegraphics[width=1\linewidth]{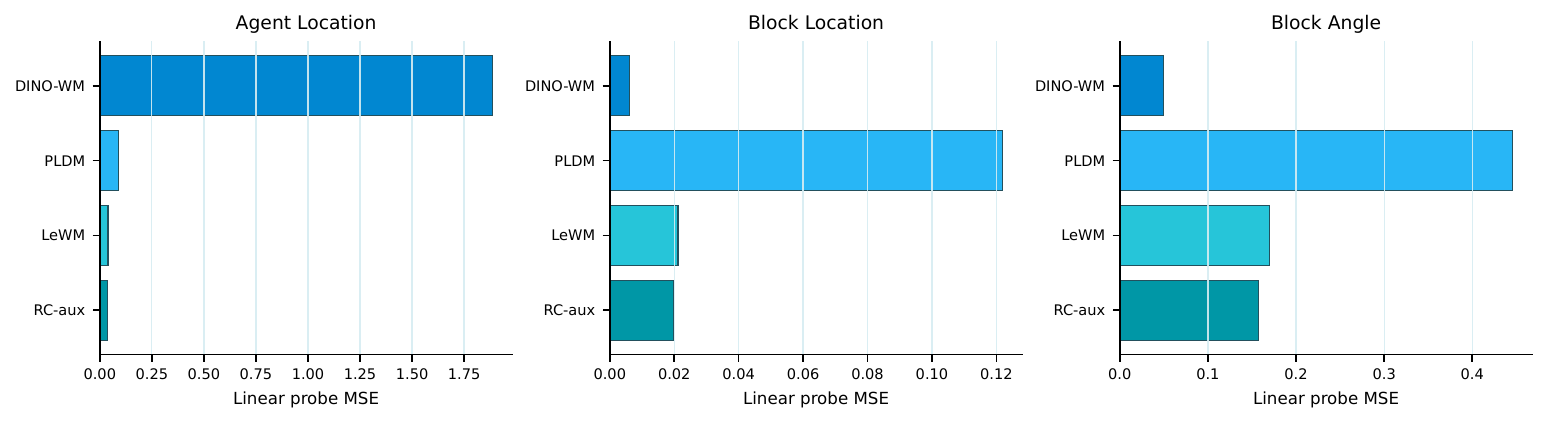}
    \caption{
    Push-T physical probe with linear predictors.
    RC-aux remains comparable to LeWM on physical quantities, indicating that reachability supervision does not remove basic physical structure from the learned latents.
    }
    \label{fig:physical-probe-pusht}
\end{figure}

\section{Broader Impacts}
\label{app:impact}

The work is foundational research on latent world models for planning. Potential positive impacts include improving the reliability and efficiency of model-based planning from pixels, while potential risks include inappropriate deployment of learned planning systems in real-world robotic settings without adequate validation, safety constraints, or uncertainty handling. The paper frames the method as an offline benchmark study rather than a deployed autonomous system.


\clearpage
\section*{NeurIPS Paper Checklist}

\begin{enumerate}

\item {\bf Claims}
    \item[] Question: Do the main claims made in the abstract and introduction accurately reflect the paper's contributions and scope?
    \item[] Answer: \answerYes{} 
    \item[] Justification: The abstract and introduction state that RC-aux improves the plannability of reconstruction-free latent world models by aligning open-loop prediction and finite-horizon reachability with downstream planning. These claims are supported by the method formulation in Section~3, the planning-alignment analysis in Appendix~B, and the controlled empirical results in Section~4 and Appendices~D--F.

\item {\bf Limitations}
    \item[] Question: Does the paper discuss the limitations of the work performed by the authors?
    \item[] Answer: \answerYes{} 
    \item[] Justification: Section~\ref{sec:limitations} discusses trajectory-derived supervision, soft feasibility correction, and the scope of the empirical evaluation.

\item {\bf Theory assumptions and proofs}
    \item[] Question: For each theoretical result, does the paper provide the full set of assumptions and a complete (and correct) proof?
    \item[] Answer: \answerYes{} 
    \item[] Justification: The paper provides the formal planning-alignment analysis in Appendix~B. The assumptions for the rollout-error and cost-distortion results are stated explicitly, and the corresponding lemmas and derivations are provided to justify how multi-horizon prediction, budget-conditioned reachability, and temporal hard negatives align the training objective with finite-horizon planning queries.

    \item {\bf Experimental result reproducibility}
    \item[] Question: Does the paper fully disclose all the information needed to reproduce the main experimental results of the paper to the extent that it affects the main claims and/or conclusions of the paper (regardless of whether the code and data are provided or not)?
    \item[] Answer: \answerYes{} 
    \item[] Justification: The method is fully specified in Section~3, including the multi-horizon loss, reachability supervision, training objective, and planner scoring rule. Section~4 describes the evaluation tasks and main protocols, while Appendix~C provides training, evaluation, planner, and LIBERO-Goal implementation details needed to reproduce the main experiments.

\item {\bf Open access to data and code}
    \item[] Question: Does the paper provide open access to the data and code, with sufficient instructions to faithfully reproduce the main experimental results, as described in supplemental material?
    \item[] Answer: \answerNo{}{} 
    \item[] Justification: Code is not included in the initial anonymous submission. The paper provides the full algorithmic formulation and detailed experimental settings in the main text and appendices, and we plan to release an anonymized implementation and reproduction instructions upon acceptance.

\item {\bf Experimental setting/details}
    \item[] Question: Does the paper specify all the training and test details (e.g., data splits, hyperparameters, how they were chosen, type of optimizer) necessary to understand the results?
    \item[] Answer: \answerYes{} 
    \item[] Justification: Section~4 specifies the evaluated goal-conditioned control tasks, baselines, matched LeWM-family comparisons, success-rate metric, and compute setting. Appendix~C further provides the training and evaluation configurations, planner settings, and LIBERO-Goal OFT-style action-head details.
    \item[] Guidelines:
    \begin{itemize}
        \item The answer \answerYes{} means that the paper does not include experiments.
        \item The experimental setting should be presented in the core of the paper to a level of detail that is necessary to appreciate the results and make sense of them.
        \item The full details can be provided either with the code, in appendix, or as supplemental material.
    \end{itemize}

\item {\bf Experiment statistical significance}
    \item[] Question: Does the paper report error bars suitably and correctly defined or other appropriate information about the statistical significance of the experiments?
    \item[] Answer: \answerYes{} 
    \item[] Justification: The main LeWM-family results are reported as mean and standard deviation over five fixed evaluation groups. The paper also includes paired fixed-episode comparisons in Appendix~E to verify that the aggregate improvements are not due to differences in evaluation episodes.

\item {\bf Experiments compute resources}
    \item[] Question: For each experiment, does the paper provide sufficient information on the computer resources (type of compute workers, memory, time of execution) needed to reproduce the experiments?
    \item[] Answer: \answerYes{} 
    \item[] Justification: Section~4 states that all experiments were conducted on a single NVIDIA RTX A6000 Ada GPU. The paper further reports parameter footprint and controlled planner cost-call timing in Section~4.4 and Appendix~E, including the overhead of RC-aux relative to LeWM and the scoring-time comparison with DINO-WM-S/14.
    
\item {\bf Code of ethics}
    \item[] Question: Does the research conducted in the paper conform, in every respect, with the NeurIPS Code of Ethics \url{https://neurips.cc/public/EthicsGuidelines}?
    \item[] Answer: \answerYes{} 
    \item[] Justification: The study uses existing simulated control and robot-manipulation benchmarks and does not involve human-subject research, sensitive personal data, or high-risk data collection. The authors have reviewed the NeurIPS Code of Ethics and made efforts to preserve anonymity throughout the submission.

\item {\bf Broader impacts}
    \item[] Question: Does the paper discuss both potential positive societal impacts and negative societal impacts of the work performed?
    \item[] Answer: \answerYes{} 
    \item[] Justification: The work is foundational research on latent world models for planning. Potential positive impacts include improving the reliability and efficiency of model-based planning from pixels, while potential risks include inappropriate deployment of learned planning systems in real-world robotic settings without adequate validation, safety constraints, or uncertainty handling. The paper frames the method as an offline benchmark study rather than a deployed autonomous system.
    
\item {\bf Safeguards}
    \item[] Question: Does the paper describe safeguards that have been put in place for responsible release of data or models that have a high risk for misuse (e.g., pre-trained language models, image generators, or scraped datasets)?
    \item[] Answer: \answerNA{} 
    \item[] Justification: The paper does not release a high-risk pretrained language model, image generator, scraped dataset, or similarly dual-use asset. The proposed method is an auxiliary objective and planner modification for latent world models evaluated on established control benchmarks, so special release safeguards are not applicable.

\item {\bf Licenses for existing assets}
    \item[] Question: Are the creators or original owners of assets (e.g., code, data, models), used in the paper, properly credited and are the license and terms of use explicitly mentioned and properly respected?
    \item[] Answer: \answerYes{} 
    \item[] Justification: The paper cites the existing benchmarks, models, and baselines used in the experiments, including LeWorldModel, DINO-WM, PLDM, and LIBERO-Goal. The experiments are based on established research assets rather than newly scraped or repackaged datasets.

\item {\bf New assets}
    \item[] Question: Are new assets introduced in the paper well documented and is the documentation provided alongside the assets?
    \item[] Answer: \answerNA{} 
    \item[] Justification: The paper does not introduce a new dataset or benchmark asset. The proposed contribution is a training objective and planner family for latent world models; the implementation will be documented when released.

\item {\bf Crowdsourcing and research with human subjects}
    \item[] Question: For crowdsourcing experiments and research with human subjects, does the paper include the full text of instructions given to participants and screenshots, if applicable, as well as details about compensation (if any)? 
    \item[] Answer: \answerNA{} 
    \item[] Justification: This study does not involve crowdsourcing experiments, human-subject experiments, or human participant data collection.

\item {\bf Institutional review board (IRB) approvals or equivalent for research with human subjects}
    \item[] Question: Does the paper describe potential risks incurred by study participants, whether such risks were disclosed to the subjects, and whether Institutional Review Board (IRB) approvals (or an equivalent approval/review based on the requirements of your country or institution) were obtained?
    \item[] Answer: \answerNA{} 
    \item[] Justification: This study does not involve human-subject research or crowdsourcing and therefore does not require IRB approval or equivalent review.

\item {\bf Declaration of LLM usage}
    \item[] Question: Does the paper describe the usage of LLMs if it is an important, original, or non-standard component of the core methods in this research? Note that if the LLM is used only for writing, editing, or formatting purposes and does \emph{not} impact the core methodology, scientific rigor, or originality of the research, declaration is not required.
    \item[] Answer: \answerNA{} 
    \item[] Justification: The core methodology does not involve large language models. RC-aux is based on latent world-model training, multi-horizon open-loop prediction, budget-conditioned reachability supervision, and reachability-aware planning; therefore, no LLM usage declaration is required for the scientific method.

\end{enumerate}

\end{document}